\documentclass[11pt]{article}

\usepackage{fullpage}
\usepackage{amsmath,amsfonts,amsthm,amssymb}
\usepackage{url}
\usepackage{color}
\usepackage[usenames,dvipsnames,svgnames,table]{xcolor}
\usepackage[colorlinks=true, linkcolor=red, urlcolor=blue, citecolor=blue]{hyperref}
\usepackage{booktabs}
\usepackage{algorithm}
\usepackage{algorithmic}
\usepackage{comment}
\usepackage{mathtools}
\usepackage{microtype}
\usepackage{caption}
\usepackage{float}
\usepackage{graphicx}
\usepackage{subfigure}
\usepackage{multirow}
\usepackage{mathrsfs}
\usepackage{newfloat}
\usepackage{relsize}
\usepackage{dsfont}
\usepackage{booktabs,siunitx,caption}
\usepackage{enumitem}
\usepackage{nicefrac}

\DeclareMathOperator*{\argmin}{arg\,min}

\numberwithin{equation}{section}
\numberwithin{figure}{section}
\theoremstyle{plain}
	\newtheorem{theorem}{Theorem}[section]

	\newtheorem{lemma}[theorem]{Lemma}

\theoremstyle{definition}
	\newtheorem{definition}[theorem]{Definition}
	
	\newtheorem*{remark*}{Remark}

\usepackage[framemethod=TikZ]{mdframed}
\newcounter{Frame}

{\qed}

\let\citet\cite
\let\citep\cite

\title{BANANAS: Bayesian Optimization with Neural Architectures
for Neural Architecture Search}

\author{
Colin White
\\
Abacus.AI\\
\texttt{colin@abacus.ai}
\and 
Willie Neiswanger
\\
Stanford University 
and Petuum Inc.\\
\texttt{neiswanger@cs.stanford.edu}
\\
\and
Yash Savani
\\
Abacus.AI\\
\texttt{yash@abacus.ai}
}

\date{}

\begin{document}

\maketitle

\begin{abstract}

Over the past half-decade, many methods have been considered for 
neural architecture search (NAS). Bayesian optimization (BO),
which has long had success in hyperparameter optimization,
has recently emerged as a very promising strategy for NAS
when it is coupled with a neural predictor.
Recent work has proposed different instantiations of this framework,
for example, using Bayesian neural networks or graph convolutional
networks as the predictive model within BO.
However, the analyses in these papers often focus on the
full-fledged NAS algorithm,
so it is difficult to tell which individual components
of the framework lead to the best performance.

In this work, we give a thorough analysis of the 
``BO + neural predictor'' framework
by identifying five main components:
the architecture encoding, neural predictor,
uncertainty calibration method, acquisition function, 
and acquisition optimization strategy.
We test several different methods for each component and
also develop a novel path-based encoding scheme for neural architectures,
which we show theoretically and empirically scales better than other
encodings.
Using all of our analyses, we develop a final algorithm called BANANAS,
which achieves state-of-the-art performance on NAS search spaces.
We adhere to the NAS research checklist (Lindauer and Hutter 2019) to 
facilitate best practices, and our code is available at 
\url{https://github.com/naszilla/naszilla}.

\end{abstract}

\section{Introduction}
\label{sec:introduction}

Since the deep learning revolution in 2012, neural networks have been growing increasingly more complex and specialized \cite{alexnet, densenet, szegedy2017inception}.
Developing new state-of-the-art architectures often takes a vast amount of engineering and domain knowledge.
A rapidly developing area of research, neural architecture search (NAS), 
seeks to automate this process.
Since the popular work by 
Zoph and Le~\citep{zoph2017neural}, 
there has been a flurry of research on NAS
\cite{pnas, enas, darts, nasbot, nas-survey, auto-keras}.
Many methods have been proposed, including evolutionary search, 
reinforcement learning, Bayesian optimization (BO), and gradient descent.
In certain settings, zeroth-order (non-differentiable) algorithms
such as BO are of particular interest over first-order (one-shot)
techniques, due to advantages such as simple parallelism, 
joint optimization with other hyperparameters,
easy implementation, portability to diverse architecture spaces,
and optimization of other/multiple non-differentiable objectives.

BO with Gaussian processes (GPs) has had success in deep learning 
hyperparameter optimization \cite{vizier, bohb},
and is a leading method for efficient zeroth order
optimization of expensive-to-evaluate functions in Euclidean spaces.
However, initial approaches for applying GP-based BO to NAS came with
challenges that limited its ability to achieve state-of-the-art results. 
For example, initial approaches required specifying a distance function
between architectures, which involved cumbersome hyperparameter 
tuning~\cite{nasbot, auto-keras}, and required a time-consuming matrix
inversion step.

Recently, Bayesian optimization with a neural
predictor has emerged as a high-performing framework for NAS.
This framework avoids the aforementioned problems with BO in NAS:
there is no need to construct a distance function between architectures,
and the neural predictor scales far better than a GP model.
Recent work has proposed different instantiations
of this framework, for example, Bayesian neural networks with 
BO~\cite{springenberg2016bayesian}, and
graph neural networks with BO~\cite{shi2019multi, ma2019deep}.
However, the analyses often focus on the 
full-fledged NAS algorithm,
making it challenging to tell which components of the
framework lead to the best performance.

In this work, we start by performing a thorough analysis of the 
``BO + neural predictor'' framework.
We identify five major components of the framework: architecture encoding, 
neural predictor,
uncertainty calibration method, acquisition function, 
and acquisition optimization strategy.
For example, graph convolutional networks, 
variational autoencoder-based networks,
or feedforward networks can be used for the neural predictor, and
Bayesian neural networks or different types of ensembling methods can be used 
for the uncertainty calibration method.
After conducting experiments on all components of the BO + neural predictor
framework, we use this analysis to define a high-performance instantiation
of the framework, which we call
BANANAS: Bayesian optimization with
neural architectures for NAS.

In order for the neural predictor to achieve the highest accuracy, 
we also define a novel path-based architecture encoding,
which we call the path encoding.
The motivation for the path encoding is as follows.
Each architecture in the search space can be represented as a labeled 
directed acyclic graph (DAG)  -- a set of nodes and directed edges,
together with a list of the operations that each node (or edge) represents.
However, the adjacency matrix can be difficult for the neural network to
interpret~\cite{zhou2018graph}, since the features are highly dependent on
one another. By contrast, each feature in our path encoding scheme represents
a unique path that the tensor can take from the input layer to the output
layer of the  architecture.
We show theoretically and experimentally that this encoding 
scales better than the adjacency matrix encoding, and allows neural predictors
to achieve higher accuracy.

We compare BANANAS to a host of popular NAS algorithms including
random search~\cite{randomnas},
DARTS~\cite{darts},
regularized evolution~\cite{real2019regularized}, 
BOHB~\cite{bohb},
NASBOT~\cite{nasbot},
local search~\cite{white2020local},
TPE~\cite{tpe},
BONAS~\cite{shi2019multi},
BOHAMIANN~\cite{springenberg2016bayesian},
REINFORCE~\cite{reinforce}, 
GP-based BO~\cite{snoek2012practical}, 
AlphaX~\cite{alphax}, 
ASHA~\cite{randomnas}, 
GCN Predictor~\cite{wen2019neural}, and
DNGO~\cite{snoek2015scalable}.
BANANAS achieves state-of-the-art performance on NASBench-101 and is competitive
on all NASBench-201 datasets.
Subsequent work has also shown that BANANAS is competitive on 
NASBench-301~\citep{nasbench301},
even when compared to first-order methods such as
DARTS~\cite{darts}, PC-DARTS~\cite{pcdarts}, and GDAS~\cite{gdas}.

Finally, to promote reproducibility, in
Appendix~\ref{app:checklist} 
we discuss how our experiments
adhere to the NAS best practices checklist~\cite{lindauer2019best}.
In particular, we experiment on well-known search spaces and NAS pipelines,
run enough trials to reach statistical significance,
and release our code.

\noindent\textbf{Our contributions.}
We summarize our main contributions.
\begin{itemize}[topsep=0pt, itemsep=2pt, parsep=0pt, leftmargin=5mm]
    \item 
    We analyze a simple framework for NAS: Bayesian optimization with a
    neural predictor, and we thoroughly test five components:
    the encoding, neural predictor, calibration,
    acquisition function, and acquisition optimization.
    \item 
    We propose a novel path-based encoding for architectures,
    which improves the accuracy of neural predictors.
    We give theoretical and experimental results showing 
    that the path encoding scales better than the 
    adjacency matrix encoding.
    \item 
    We use our analyses to develop BANANAS, a high performance instantiation
    of the above framework.
    We empirically show that BANANAS is state-of-the-art on popular
    NAS benchmarks.
\end{itemize}

\section{Societal Implications} \label{sec:impact}
Our work gives a new method for neural architecture search, with the aim of
improving the performance of future deep learning research.
Therefore, we have much less control over the net impact of our work on
society. For example, our work may be used to tune a deep
learning optimizer for reducing the carbon footprint of large power plants,
but it could just as easily be used to improve a deep fake generator.
Clearly, the first example would have a positive impact on society, while
the second example may have a negative impact.

Our work is one level of abstraction from real applications,
but our algorithm, and more generally the field of NAS,
may become an important step in advancing the field of artificial intelligence.
Because of the recent push for explicitly reasoning about the impact of
research in AI~\cite{hecht2018time}, we are hopeful that neural architecture
search will be used to benefit society.

\section{Related Work}
\label{sec:relatedwork}

NAS has been studied since at least the 1990s
and has gained significant attention in the past few years
\cite{kitano1990designing, stanley2002evolving, zoph2017neural}. 
Some of the most popular recent techniques for NAS include evolutionary algorithms \cite{maziarz2018evolutionary}, reinforcement learning \cite{zoph2017neural, enas}, BO~\cite{nasbot}, and gradient descent \cite{darts}.
For a survey of neural architecture search, see~\cite{nas-survey}.

Initial BO approaches defined a distance function between 
architectures~\cite{nasbot, auto-keras}.
There are several works that predict the validation accuracy 
of neural 
networks~\cite{klein2016learning, peephole, tapas, zhang2018graph, baker2017accelerating}.
A few recent papers have used Bayesian optimization with a graph neural
network as a predictor~\cite{ma2019deep, shi2019multi},
however, they do not conduct an ablation study of all components of the 
framework.
In this work, we do not claim to invent the BO + neural predictor framework,
however, we give the most in-depth analysis that we are aware of, which we
use to design a high-performance instantiation of this framework.

There is also prior work on using neural network models in
BO for hyperparameter optimization
\cite{snoek2015scalable, springenberg2016bayesian}, 
The explicit goal of these papers is to improve the efficiency of 
Gaussian process-based BO from cubic to linear time,
not to develop a different type of prediction model in order to
improve the performance of BO with respect to the number of iterations. 

Recent papers have called for fair and reproducible experiments~\cite{randomnas,  nasbench}.
In this vein, the NASBench-101~\cite{nasbench}, -201~\cite{nasbench201},
and -301~\cite{nasbench301} datasets 
were created, which contain tens of thousands of pretrained neural architectures.
We provide additional related work details in Appendix~\ref{app:relatedwork}.

\paragraph{Subsequent work.}
Since its release, several papers have included BANANAS in new experiments,
further showing that BANANAS is a competitive NAS 
algorithm~\citep{remaade, nasbench301,nguyen2020optimal, nasbowl, npenas}.
Finally, a recent paper conducted a study on several encodings used for 
NAS~\cite{white2020study}, concluding that neural predictors perform well
with the path encoding.

\section{BO + Neural Predictor Framework}
\label{sec:preliminaries}
In this section, we give a background on BO, and we describe the
BO + neural predictor framework.
In applications of BO for deep learning, the typical goal is to find a
neural architecture and/or set of hyperparameters that lead to
an optimal validation error. Formally, BO
seeks to compute $a^* = \argmin_{a\in A} f(a)$, where $A$ is
the search space, and $f(a)$
denotes the validation error of architecture $a$ after training
on a fixed dataset for a fixed number of epochs. In the standard 
BO setting,
over a sequence of iterations, the results from all previous iterations 
are used to model the topology of $\{f(a)\}_{a\in A}$ using
the posterior distribution of the model (often a GP).
The next architecture is then chosen by optimizing an acquisition
function such as expected improvement (EI) \cite{movckus1975bayesian}
or Thompson sampling (TS) \cite{thompson1933likelihood}.
These functions balance exploration with exploitation during the search.
The chosen architecture is then trained and used to update the model of $\{f(a)\}_{a\in A}$.
Evaluating $f(a)$ in each iteration is the bottleneck of BO 
(since a neural network must be trained). To mitigate this,
parallel BO methods typically output $k$ architectures to
train in each iteration, so that the $k$
architectures can be trained in parallel.

\paragraph{BO + neural predictor framework.}
In each iteration of BO, we train a neural network
on all previously evaluated architectures, $a$, to predict the validation
accuracy $f(a)$ of unseen architectures.
The architectures are represented as labeled DAGs~\cite{nasbench, nasbench201},
and there are different methods of encoding the DAGs before they are passed to
the neural predictor~\cite{nasbench, white2020study}, 
which we describe in the next section.
Choices for the neural predictor include feedforward networks, 
graph convolutional networks (GCN),
and variational autoencoder (VAE)-based networks.
In order to evaluate an acquisition function, we also compute an uncertainty
estimate for each input datapoint. This can be accomplished by using,
for example, a Bayesian neural network or an ensemble of neural predictors.
Given the acquisition function, an acquisition optimization routine is then 
carried out, which returns the next architecture to be evaluated.
In the next section, we give a thorough analysis of the choices that
must be made when instantiating this framework.

\section{Analysis of the Framework}
\label{sec:methodology}

\begin{figure*}[h]
\centering
\includegraphics[width=0.5\textwidth]{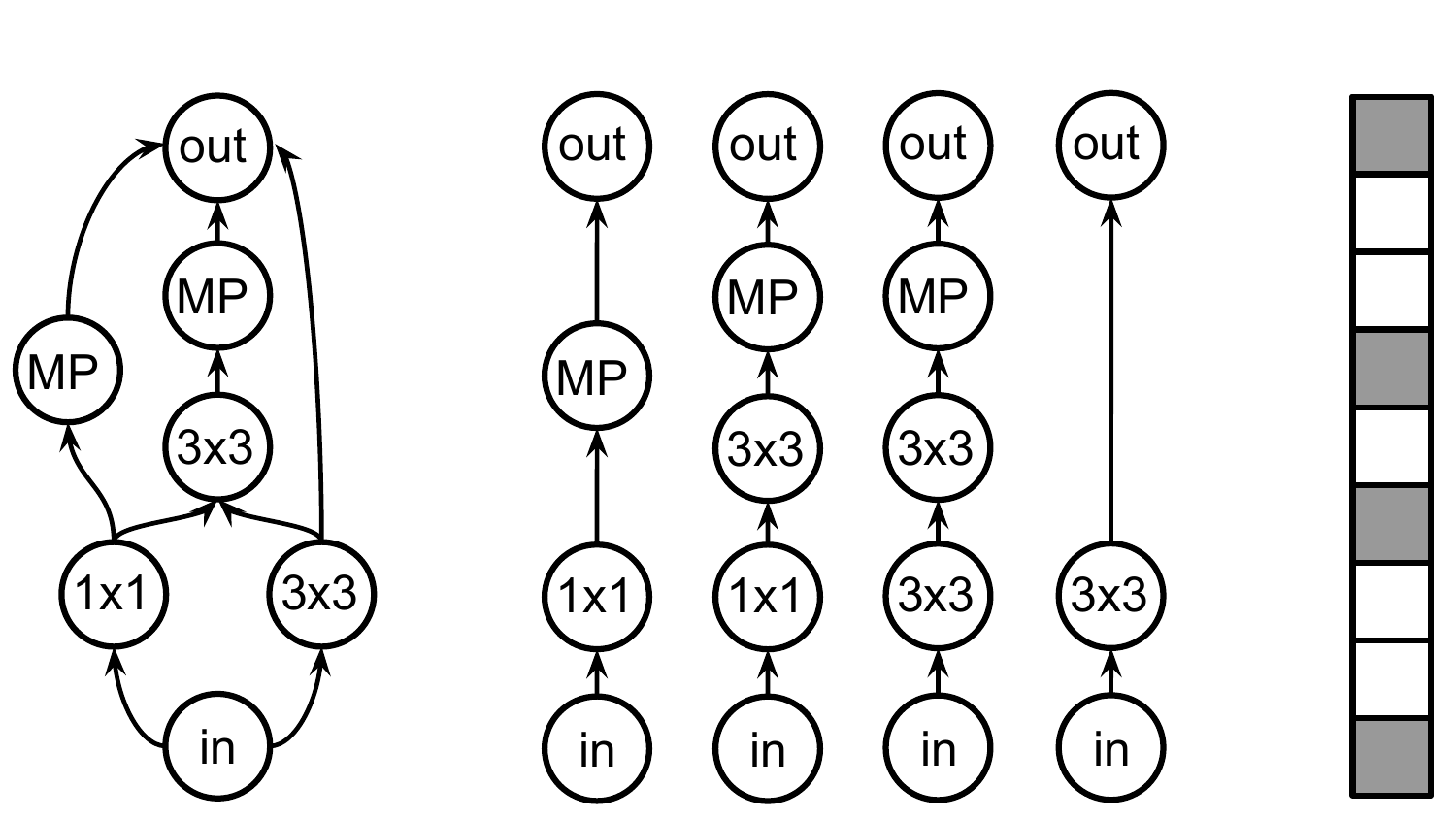}
\hspace{1cm}
\includegraphics[width=0.38\textwidth]{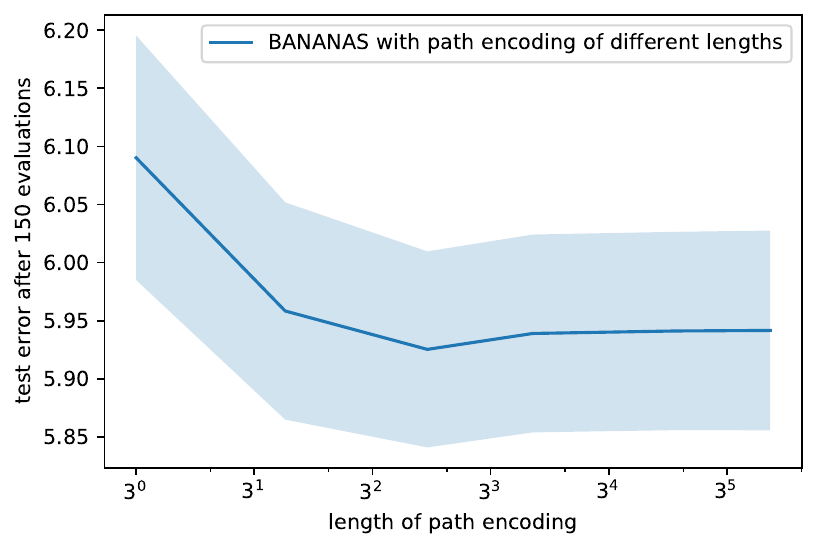}
\caption{Example of the path encoding (left).
Performance of BANANAS with the path encoding truncated to different
lengths (right). Since each node has 3 choices of operations, the ``natural''
cutoffs are at powers of 3.
}
\label{fig:path_encoding}
\end{figure*}

In this section, we give an extensive study of the BO + neural predictor
framework. First, we discuss architecture encodings, and we define a novel
featurization called the path encoding.
Then we conduct an analysis of different choices of neural predictors.
Next, we analyze different methods for achieving calibrated uncertainty
estimates from the neural predictors.
After that, we conduct experiments on different acquisition functions
and acquisition optimization strategies.
Finally, we use these analyses to create our algorithm, BANANAS.

Throughout this section, we run experiments on the NASBench-101 dataset
(experiments on additional search spaces are given in Section~\ref{sec:experiments}).
The NASBench-101 dataset~\citep{nasbench} consists of over 423,000 
neural architectures from a cell-based search space,
and each architecture comes with precomputed validation 
and test accuracies on CIFAR-10. 
The search space consists of a DAG with 7 nodes
that can each take on three different operations, and there can
be at most 9 edges between the nodes. 
We use the open source version of the NASBench-101 dataset~\cite{nasbench}.
We give the full details about the use of NASBench-101 in
Appendix~\ref{app:experiments}. Our code is available at
\url{https://github.com/naszilla/naszilla}.


\paragraph{Architecture encodings.}
The majority of existing work on neural predictors use an adjacency
matrix representation to encode the neural architectures.
The adjacency matrix encoding gives an arbitrary ordering to the nodes, 
and then gives a binary feature for an edge between node $i$ and node $j$, 
for all $i<j$. 
Then a list of the operations at each node must also be included in the encoding.
This is a challenging data structure for a neural predictor to interpret 
because it relies on an arbitrary indexing of the nodes, and features are 
highly dependent on one another. For example, an edge from the input to node 2 
is useless if there is no path from node 2 to the output.
And if there \emph{is} an edge from node 2 to the output, this edge is highly
correlated with the feature that describes the operation at node 2 
(conv\_1x1, pool\_3x3, etc.).
A continuous-valued variant of the adjacency matrix encoding has also 
been tested~\cite{nasbench}.

We introduce a novel encoding which we term the \emph{path encoding},
and we show that it substantially increases the performance of neural predictors.
The path encoding is quite simple to define: there is a binary feature for each
path from the input to the output of an architecture cell,
given in terms of the operations
(e.g., input$\rightarrow$conv\_1x1$\rightarrow$pool\_3x3$\rightarrow$output). 
To encode an architecture, we simply check which paths are present in the architecture,
and set the corresponding features to 1s.
See Figure \ref{fig:path_encoding}.
Intuitively, the path encoding has a few strong advantages.
The features are not nearly as dependent on one another as they are in the adjacency
matrix encoding, since each feature represents a unique path that the data tensor can
take from the input node to the output node.
Furthermore, there is no longer an arbitrary node ordering, which means that
each neural architecture maps to only one encoding (which is not true for the
adjacency matrix encoding). On the other hand, it is possible for multiple
architectures to map to the same path encoding
(i.e., the encoding is well-defined, but it is not one-to-one).
However, subsequent work showed that architectures
with the same path encoding also have very similar validation 
errors~\cite{white2020study}, which is beneficial in NAS algorithms.

The length of the path encoding is the total number of possible paths in a cell,
$\sum_{i=0}^n q^i$, where $n$ denotes
the number of nodes in the cell, and $q$ denotes the number
of operations for each node. 
However, we present theoretical and experimental evidence that substantially
truncating the path encoding, even to length smaller than the adjacency matrix encoding,
does not decrease its performance.
Many NAS algorithms sample architectures by randomly sampling edges in the DAG 
subject to a maximum edge constraint~\cite{nasbench}.
Intuitively, the vast majority of paths have a very low probability of occurring
in a cell returned by this procedure.
Therefore, by simply truncating the least-likely paths,
our encoding scales \emph{linearly} in the size of the cell,
with an arbitrarily small amount of information loss.
In the following theorem, let $G_{n, k, r}$ denote a DAG architecture with $n$ nodes,
$r$ choices of operations on each node, and where each potential forward edge 
($\frac{n(n-1)}{2}$ total) was chosen with probability $\frac{2k}{n(n-1)}$ 
(so that the expected number of edges is $k$).

\begin{theorem}[\textbf{informal}] \label{thm:path_length_informal}
Given integers $r, c>0$, there exists an $N$ such that
$\forall~n>N$, there exists a set of $n$ paths $\mathcal{P}'$
such that the probability that $G_{n,n+c,r}$ contains
a path not in $\mathcal{P}'$ is less than $\frac{1}{n^{2}}$.
\end{theorem}

For the formal statement and full proof, see 
Appendix~\ref{app:methodology}.
This theorem says that when $n$ is large enough, with high probability, 
we can truncate the path encoding to a size of just $n$ without losing information.
Although the asymptotic nature of this result makes it a proof of concept, 
we empirically show in Figure~\ref{fig:path_encoding} that in BANANAS running on
NASBench-101,
the path encoding can be truncated from its full size of 
$\sum_{i=0}^5 3^i=364$ bits to a length of just \emph{twenty} bits,
without a loss in performance. (The exact experimental setup for this result is 
described later in this section.)
In fact, the performance after truncation actually \emph{improves} up to a certain
point. We believe this is because with the full-length encoding, 
the neural predictor overfits to very rare paths.
In Appendix~\ref{app:experiments},
we show a similar result for NASBench-201 \cite{nasbench201}:
the full path encoding length of $\sum_{i=0}^3 5^i=156$ can be truncated to just 30, 
without a loss of performance.

\paragraph{Neural predictors.}
Now we study the neural predictor, a crucial component in the BO + neural predictor
framework. 
Recall from the previous section that a neural predictor is a neural network 
that is repeatedly trained on the current set of evaluated 
neural architectures and predicts the
accuracy of unseen neural architectures. 
Prior work has used GCNs~\cite{shi2019multi, ma2019deep}
or VAE-based architectures~\cite{dvae} for this task.
We evaluate the performance of standard feedfoward
neural networks with either the adjacency matrix or path-based encoding,
compared to VAEs and GCNs in predicting the validation accuracy 
of neural architectures.
The feedforward neural network we use is a sequential fully-connected
network with 10 layers of width 20, 
the Adam optimizer with a learning rate of $0.01$, 
and the loss function set to mean absolute error (MAE).
We use open-source implementations of the GCN~\cite{zhang2020neural}
and VAE~\cite{dvae}. 
See Appendix~\ref{app:experiments}
for a full 
description of our implementations.

In Figure~\ref{fig:neural_predictor} (left), we compare the different neural
predictors by training them on a set of neural architectures drawn i.i.d.\ from
NASBench-101, along with validation accuracies, and then computing the 
MAE on a held-out test set of size 1000. 
We run 50 trials for different training set sizes
and average the results.
The best-performing neural predictors are the feedforward network with the path encoding
(with and without truncation) and the GCN.
The feedforward networks also had shorter runtime compared to the GCN and VAE,
however, the runtime of the full NAS algorithm is dominated by evaluating neural architectures,
not by training neural predictors.

\begin{figure*}[ht]
\centering
\includegraphics[width=0.32\textwidth]{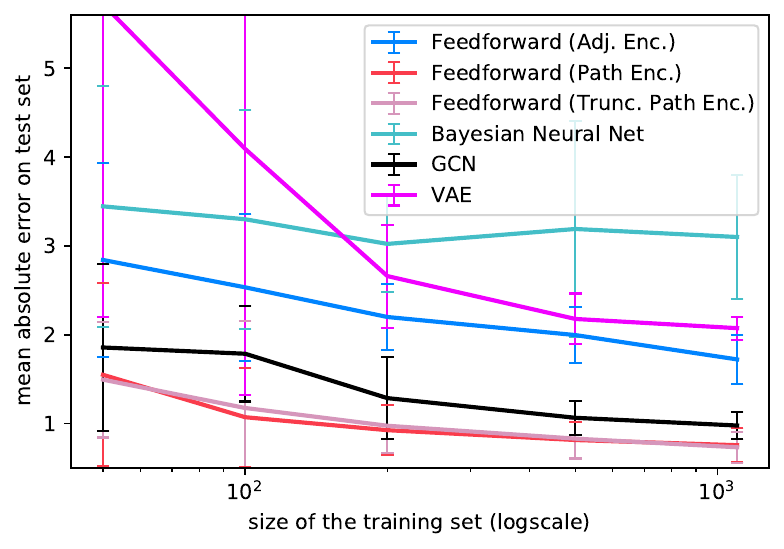}
\includegraphics[width=0.32\textwidth]{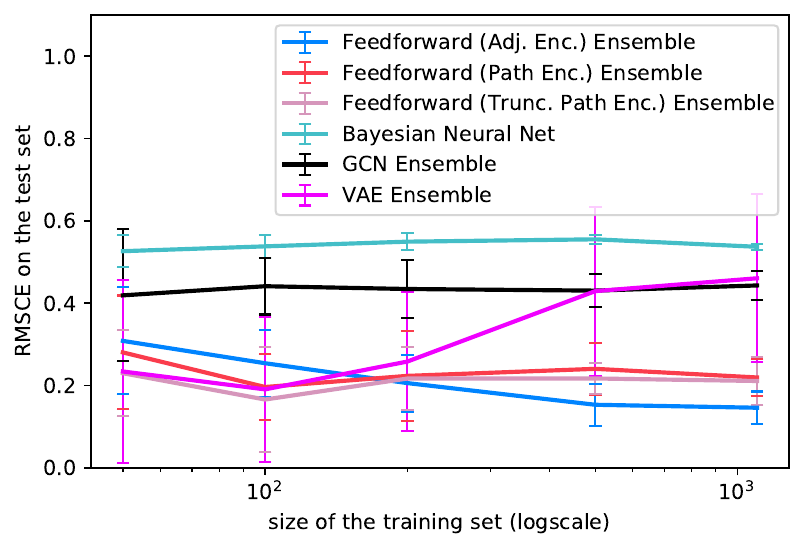}
\includegraphics[width=0.32\textwidth]{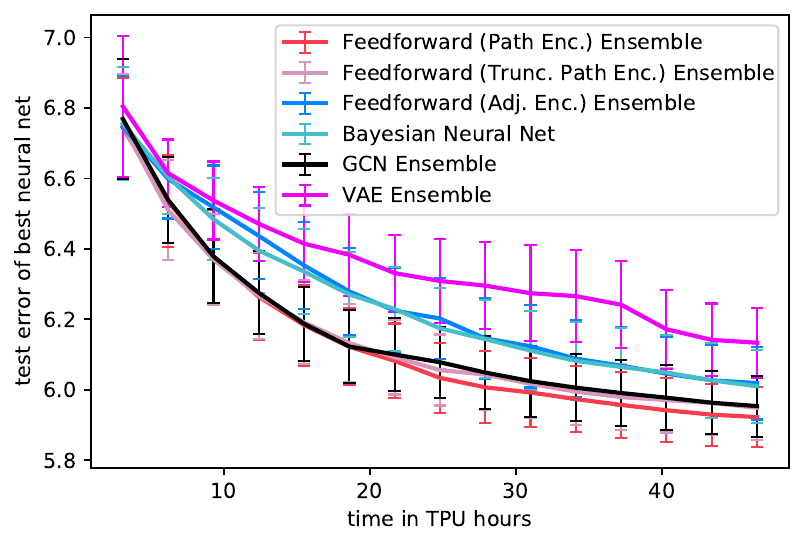}
\caption{Performance of neural predictors on NASBench-101:
predictive ability (left), accuracy of uncertainty estimates (middle),
performance in NAS when combined with BO (right). 
}
\label{fig:neural_predictor}
\end{figure*}


\paragraph{Uncertainty calibration.}
In the previous section, we evaluated standalone neural predictors. 
To incorporate them within BO, for any datapoint, neural predictors need to output 
both a prediction and an uncertainty estimate for that prediction.
Two popular ways of achieving uncertainties are by using a Bayesian neural
network (BNN), or by using an ensemble of neural predictors.
In a BNN, we infer a posterior distribution over network weights.
It has been demonstrated recently that accurate prediction and uncertainty 
estimates in neural networks can be achieved using Hamiltonian 
Monte Carlo~\cite{springenberg2016bayesian}.
In the ensemble approach, we train $m$ neural predictors 
using different random weight initializations and training set orders.
Then for any datapoint, we can can compute the mean and standard
deviation of these $m$ predictions.
Ensembles of neural networks, even of size three and five,
have been shown in some cases to give more reliable uncertainty 
estimates than other leading approaches such as BNNs~\cite{lakshminarayanan2017simple, beluch2018power, choi2016ensemble, snoek2019can, zaidi2020neural}.

We compare the uncertainty estimate of a BNN with an ensemble of size
five for each of the neural predictors described in the previous section.
We use the BOHAMIANN implementation for the BNN~\cite{springenberg2016bayesian}, 
and to ensure a fair comparison with the ensembles, 
we train it for five times longer.
The experimental setup is similar to the previous section, but we compute
a standard measure of calibration:
root mean squared calibration error (RMSCE) on the test set
\cite{kuleshov2018accurate, tran2020methods}.
See Figure~\ref{fig:neural_predictor} (middle).
Intuitively, the RMSCE is low if a method yields a well-calibrated predictive estimate
(i.e. predicted coverage of intervals equals the observed coverage).
All ensemble-based predictors yielded better uncertainty estimates than 
the BNN, consistent with prior work. 
Note that RMSCE only measures the quality of uncertainty estimates, agnostic to 
prediction accuracy.
We must therefore look at prediction (Figure~\ref{fig:neural_predictor} left) and
RMSCE (Figure~\ref{fig:neural_predictor} middle) together when evaluating the neural predictors.

Finally, we evaluate the performance of each neural predictor within the full
BO + neural predictor framework. We use the approach described in
Section~\ref{sec:preliminaries}, using independent Thompson sampling and mutation
for acquisition optimization (described in more detail in the next section).
Each algorithm is given a budget of 47 TPU hours, or about 150 neural architecture
evaluations on NASBench-101. That is, there are 150 iterations of training a
neural predictor and choosing a new architecture to evaluate using the
acquisition function.
The algorithms output 10 architectures in each iteration of BO for better
parallelization, as described in the previous section.
After each iteration, we return the test error of the architecture with
the best validation error found so far. We run 200 trials of each algorithm and
average the results. 
This is the same experimental setup as in Figure~\ref{fig:path_encoding},
as well as experiments later in this section and the next section.
See Figure~\ref{fig:neural_predictor} (right).
The two best-performing neural predictors are an ensemble of GCNs, and an
ensemble of feedforward neural networks with the path encoding, with the
latter having a slight edge. 
The feedforward network is also desirable because it
requires less hyperparameter tuning than the GCN.


\paragraph{Acquisition functions and optimization.}

Now we analyze the BO side of the framework, namely, 
the choice of acquisition function and acquisition optimization.
We consider four common acquisition functions that can be computed using a mean and 
uncertainty estimate for each input datapoint:
expected improvement (EI) \cite{movckus1975bayesian},
probability of improvement (PI) \cite{kushner1964new},
upper confidence bound (UCB) \cite{srinivas2009gaussian},
and Thompson sampling (TS) \cite{thompson1933likelihood}.
We also consider a variant of TS called independent Thompson sampling.
First we give the formal definitions of each acquisition function.

Suppose we have trained an ensemble of $M$ predictive models,
$\{f_m\}_{m=1}^M$,
where $f_m : A \rightarrow \mathbb{R}$.
Let $y_{\min}$ denote the lowest validation error of an architecture discovered so
far.
Following previous work \cite{neiswanger2019probo},
we use the following acquisition function estimates for an
input architecture $a \in A$:
\begin{align}
    \phi_\text{EI}(a) &= \mathbb{E} \left[ 
    \mathds{1}\left[ f_m(a) > y_\text{min} \right]
    \left(y_\text{min} - f_m(a) \right) 
    \right] \\
    &= \int_{-\infty}^{y_\text{min}} 
    \left(y_\text{min} - y \right)
    \mathcal{N}\left(\hat{f}, \hat{\sigma}^2 \right)
    dy \nonumber 
\end{align}
\begin{align}    
    \phi_\text{PI}(x) &= \mathbb{E} \left[
    \mathds{1}\left[ f_m(x) > y_\text{min} \right]
    \right] \\
    &= \int_{-\infty}^{y_\text{min}} 
    \mathcal{N}\left(\hat{f}, \hat{\sigma}^2 \right)
    dy \nonumber\\
    \phi_\text{UCB}(x) &= 
    \hat{f} - \beta \hat{\sigma} \\
    \phi_\text{TS}(x) &= f_{\tilde{m}}(x),
        \hspace{2mm}
        \tilde{m} \sim \text{Unif}\left(1,M\right)\\
    \phi_\text{ITS}(x) &= \tilde{f}_x(x),
        \hspace{2mm}
        \tilde{f}_x(x) \sim \mathcal{N}(\hat{f}, \hat{\sigma}^2)
\end{align}
In these acquisition function definitions, 
$\mathds{1}(x) = 1$ if $x$ is true and $0$ 
otherwise, and we are making a normal
approximation for our model's posterior predictive
density, where we estimate parameters 
\begin{equation*}
\hat{f} = \frac{1}{M}\sum_{m=1}^M f_m(x), \text{ and }
\hat{\sigma} = \sqrt{\frac{\sum_{m=1}^M (f_m(x) - \hat{f})^2}{M-1}}. 
\end{equation*}

In the UCB acquisition function experiments, we set the tradeoff parameter $\beta=0.5$.
We tested each acquisition function within the BO + neural predictor framework,
using mutation for acquisition optimization and
the best neural predictor from the previous section - an ensemble of
feedforward networks with the path encoding. 
The experimental setup is the same as in previous sections.
See Figure~\ref{fig:nasbench_main} (left).
We see that the acquisition function does not have as big an effect on performance
as other components, though ITS performs the best overall.
Note also that both TS and ITS have advantages when running parallel
experiments, since they are stochastic acquisition functions that can be directly
applied in the batch BO setting~\cite{kandasamy2018parallelised}.

Next, we test different acquisition optimization strategies.
In each iteration of BO, our goal is to find the neural architecture from the search 
space which minimizes the acquisition function.
Evaluating the acquisition function for every neural architecture
in the search space is computationally infeasible.
Instead, we create a set of $100$-$1000$ architectures
(potentially in an iterative fashion) and choose the architecture with
the value of the acquisition function in this set.

The simplest strategy is to draw 1000 random architectures.
However, it can be beneficial to generate a set of architecture that are close in
edit distance to architectures in the training set, since the neural predictor is
more likely to give accurate predictions to these architectures.
Furthermore, local optimization methods such as mutation, evolution, and local
search have been shown to be effective for acquisition 
optimization~\citep{balandat2019botorch, nasbot, wilson2018maximizing}.
In ``mutation'', we simply mutate the architectures
with the best validation accuracy that we have found so far by randomly changing
one operation or one edge.
In local search, we iteratively take the architectures with the current highest
acquisition function value, and compute the acquisition function of all 
architectures in their neighborhood.
In evolution, we iteratively maintain a population by mutating the architectures
with the highest acquisition function value and killing the architectures with the
lowest values.
We give the full details of these methods in 
Appendix~\ref{app:experiments}.
The experimental setup is the same as in the previous sections.
See Figure~\ref{fig:nasbench_main} (middle).  
We see that mutation performs the best, which indicates that it is better to consider
architectures with edit distance closer to the set of already evaluated architectures.


\begin{figure}
\centering
    \includegraphics[width=.7\textwidth]{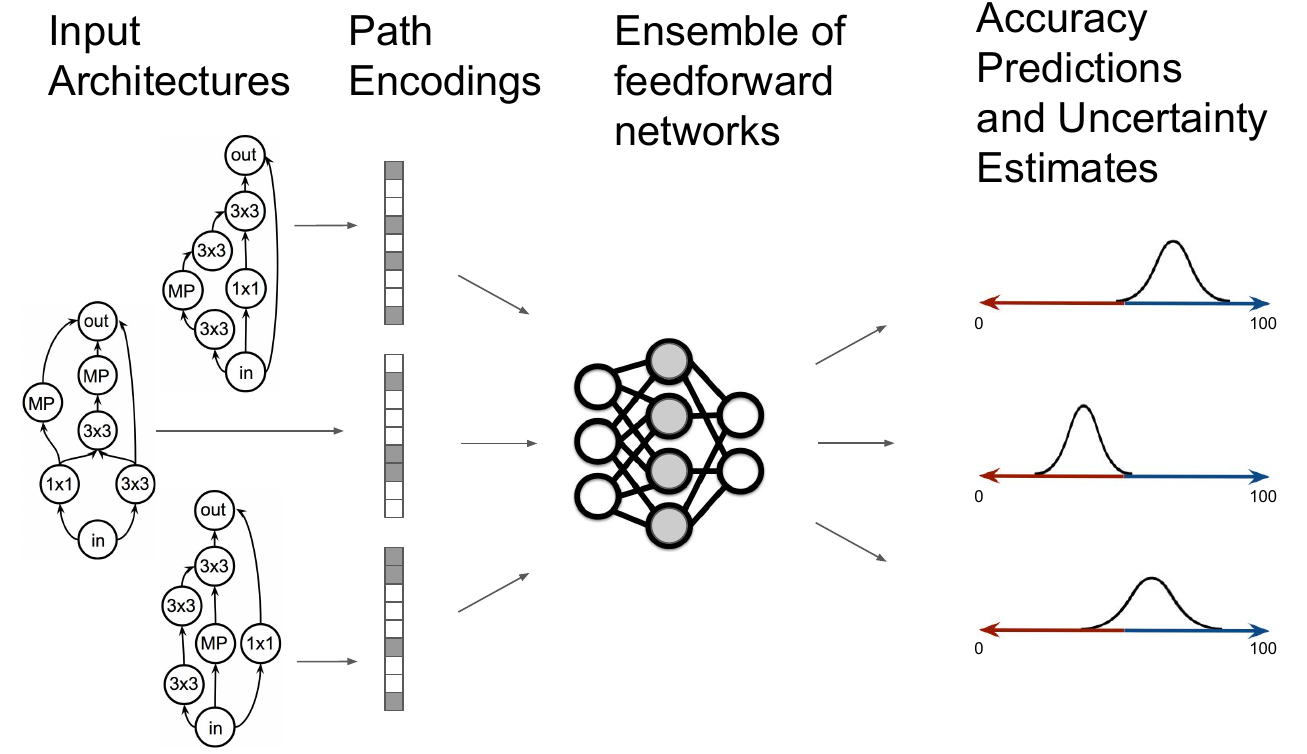}
    \caption{Diagram of the BANANAS neural predictor.}
    \label{fig:bananas_diagram}
\end{figure}

\paragraph{BANANAS: Bayesian optimization with neural architectures for NAS.}
Using the best components from the previous sections, we construct our full NAS algorithm,
BANANAS, composed of an ensemble of feedforward neural networks using the path encoding,
ITS, and a mutation acquisition function.
See Algorithm~\ref{alg:bananas} and Figure~\ref{fig:bananas_diagram}.
Note that in the previous sections, we conducted experiments on each component
individually while keeping all other components fixed.
In Appendix~\ref{app:experiments},
we give further analysis varying all components at
once, to ensure that BANANAS is indeed the optimal instantiation of this framework. 

For the loss function in the neural predictors, 
we use mean absolute percentage error (MAPE) because it gives a higher weight to
architectures with lower validation losses:
\begin{equation}
\mathcal{L}(y_\text{pred}, y_\text{true}) =
    \frac{1}{n}\sum_{i=1}^n \left| \frac{y_\text{pred}^{(i)} - y_\text{LB}}{y_\text{true}^{(i)} - y_\text{LB}} - 1\right|,\label{eq:mape}
\end{equation}
where $y_\text{pred}^{(i)}$ and $y_\text{true}^{(i)}$ are the predicted and true values of the validation error for architecture $i$, and $y_\text{LB}$ is a global lower bound on the minimum true validation error.
To parallelize Algorithm~\ref{alg:bananas}, in step iv.\ we simply choose the $k$ architectures with the smallest values of the acquisition function and evaluate the architectures in parallel.

\begin{algorithm}
\caption{BANANAS}\label{alg:bananas}
\begin{algorithmic} 
\STATE {\bfseries Input:} Search space $A$, dataset $D$, parameters $t_0,~T,~M,~c,~x$, acquisition function $\phi$, function $f(a)$ returning validation error of $a$ after training.
\STATE 1. Draw $t_0$ architectures $a_0, \dots, a_{t_0}$ uniformly at random from $A$ and train them on $D$.
\STATE 2. For $t$ from $t_0$ to $T$,
\begin{enumerate}[label=\roman*., itemsep=1pt, parsep=1mm, topsep=1pt, leftmargin=8mm]
    \item 
Train an ensemble of neural predictors on $\{ (a_0, f(a_0)), \dots, (a_t, f(a_t)) \}$
using the path encoding to represent each architecture.
\item
Generate a set of $c$ candidate architectures from $A$ by
randomly mutating the $x$ architectures $a$ from $\{a_0, \dots, a_t\}$
that have the lowest value of $f(a)$.
\item 
For each candidate architecture $a$, evaluate the acquisition function $\phi(a)$.
\item
Denote $a_{t+1}$ as the candidate architecture with minimum $\phi(a)$, and evaluate $f(a_{t+1})$.
\end{enumerate}
\STATE {\bfseries Output:} $a^*=\text{argmin}_{t=0,\ldots,T} f(a_t)$.
\end{algorithmic}
\end{algorithm}

\section{BANANAS Experiments}
\label{sec:experiments}

In this section, we compare BANANAS to many other popular NAS algorithms
on three search spaces.
To promote reproducibility, we discuss our adherence to the 
NAS research checklist~\cite{lindauer2019best} in 
Appendix~\ref{app:checklist}.
In particular, we release our code, we use
a tabular NAS dataset, and we run many trials of each algorithm.

We run experiments on NASBench-101 described in the previous section, 
as well as NASBench-201 and the DARTS search space.
The NASBench-201 dataset~\citep{yang2019evaluation} consists of
$15625$ neural architectures with precomputed validation and test
accuracies for 200 epochs on CIFAR-10, CIFAR-100, and ImageNet-16-120.
The search space consists of a complete directed acyclic graph on 4 nodes,
and each edge can take on five different operations.
The DARTS search space~\citep{darts} 
is size $10^{18}$. It consists of two cells: 
a convolutional cell and a reduction cell.
Each cell has four nodes that have two incoming edges which
take on one of eight operations. 

\begin{figure*}[t]
\centering
\includegraphics[width=0.32\textwidth]{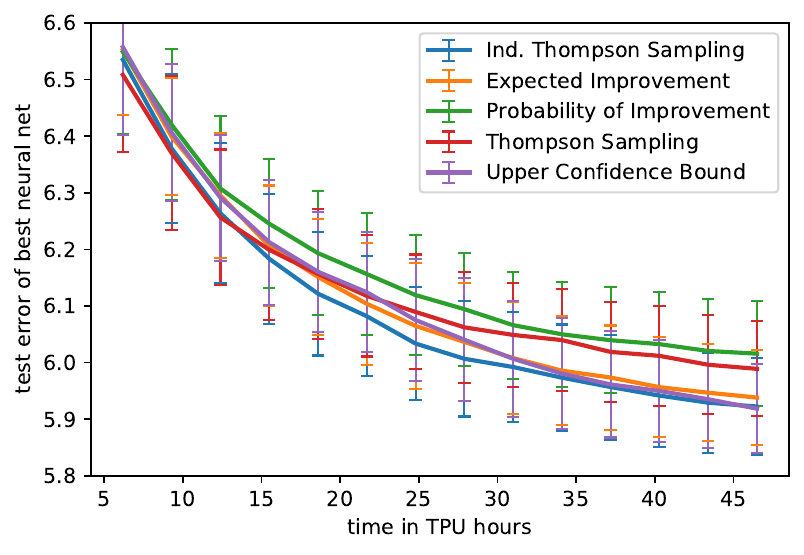}
\includegraphics[width=0.32\textwidth]{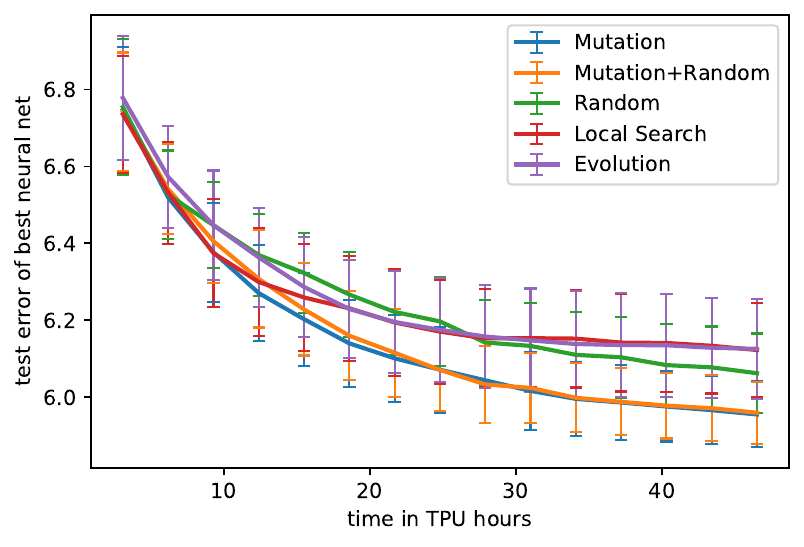}
\includegraphics[width=0.32\textwidth]{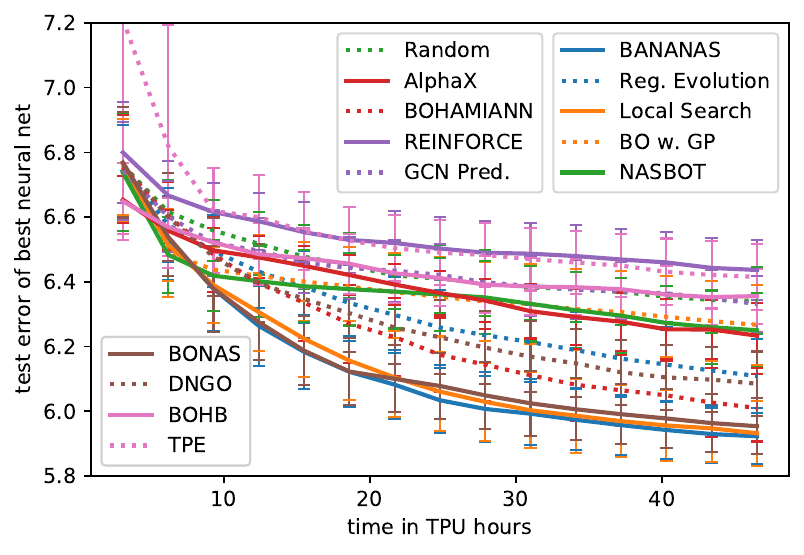}
\caption{Performance of different acquisition functions (left).
Performance of different acquisition optimization strategies (middle).
Performance of BANANAS compared to other NAS algorithms (right). 
See Appendix~\ref{app:experiments}
for the same results in a table.}
\label{fig:nasbench_main}
\end{figure*}

\begin{table*}[t]
\caption{Comparison of NAS algorithms on the DARTS search space. 
The runtime unit is total GPU-days on a Tesla V100.}
\setlength\tabcolsep{0pt}
\begin{tabular*}{\textwidth}{l @{\extracolsep{\fill}}*{8}{S[table-format=1.4]}} 
\toprule
\multicolumn{1}{c}{NAS Algorithm} & \multicolumn{1}{c}{Source} & \multicolumn{1}{c}{Avg. Test error} & \multicolumn{1}{c}{Runtime} & \multicolumn{1}{c}{Method} \\
\midrule
Random search & \cite{darts} & 3.29  & 4 & \text{Random} \\
Local search & \cite{white2020local} & 3.49  & 11.8 & \text{Local search} \\
DARTS & \cite{darts} & 2.76  & 5 & \text{Gradient-based} \\
ASHA & \cite{randomnas} & 3.03 & 9 & \text{Successive halving} \\
Random search WS & \cite{randomnas} & 2.85 & 9.7 & \text{Random} \\
\hline
DARTS & \hspace{-0.5mm}Ours & 2.68 & 5 & \text{Gradient-based} \\
ASHA & \hspace{-0.5mm}Ours & 3.08 & 9 & \text{Successive halving} \\
BANANAS & \hspace{-0.5mm}Ours &  \hspace{-3.4mm}\textbf{2.64} & 11.8 & \text{BO + neural predictor} \\
\bottomrule
\end{tabular*} 
\label{table:darts}
\end{table*} 

\paragraph{Performance on NASBench search spaces.}
We compare BANANAS to the most popular NAS algorithms from a variety of paradigms:
random search~\cite{randomnas},
regularized evolution~\cite{real2019regularized}, 
BOHB~\cite{bohb},
NASBOT~\cite{nasbot},
local search~\cite{white2020local},
TPE~\cite{tpe},
BOHAMIANN~\cite{springenberg2016bayesian},
BONAS~\cite{shi2019multi},
REINFORCE~\cite{reinforce}, 
GP-based BO~\cite{snoek2012practical}, 
AlphaX~\cite{alphax}, 
GCN Predictor~\cite{wen2019neural}, and
DNGO~\cite{snoek2015scalable}.
As much as possible, we use the code directly from the open-source repositories,
without changing the hyperparameters (but with a few exceptions).
For a description of each algorithm and 
details of the implementations we used, see 
Appendix~\ref{app:experiments}.

The experimental setup is the same as in the previous section.
For results on NASBench-101, see Figure~\ref{fig:nasbench_main} (right).
The top three algorithms in order, are BANANAS, local search, and BONAS.
In Appendix~\ref{app:experiments},
we also show that BANANAS achieves strong performance on the three
datasets in NASBench-201.

\paragraph{Performance on the DARTS search space.}
We test BANANAS on the search space from DARTS.
Since the DARTS search space is not a tabular dataset, we cannot fairly compare
to other methods which use substantially different training and testing 
pipelines~\cite{lindauer2019best}.
We use a common test evaluation pipeline which is to train for 
600 epochs with cutout and auxiliary tower~\cite{darts, randomnas, yan2020does},
where the state of the art is around 2.6\% on CIFAR-10. 
Other papers use different test evaluation settings 
(e.g., training for many more epochs) 
to achieve lower error, but they cannot be fairly
compared to other algorithms.

In our experiments,
BANANAS is given a budget of 100 evaluations. In each evaluation, 
the chosen architecture is trained for 50 epochs and the average
validation error of the last 5 epochs is recorded.
To ensure a fair comparison by controlling all hyperparameter settings
and hardware, we re-trained the architectures from prior work
when they were available.
In this case, we report the mean test error over five random seeds of the best architecture
found for each method.
We compare BANANAS to DARTS~\cite{darts}, random search~\cite{darts}, 
local search~\cite{white2020local}, and ASHA~\cite{randomnas}.
See Table~\ref{table:darts}.


Note that a new surrogate benchmark on the DARTS search space~\cite{nasbench301}, 
called NASBench-301 was recently introduced, 
allowing for fair and computationally feasible experiments.
Initial experiments showed~\cite{nasbench301}
that BANANAS is competitive with nine other popular NAS algorithms, 
including DARTS~\cite{darts} and two improvements of DARTS~\cite{pcdarts, gdas}.

\section{Conclusion and Future Work}
\label{sec:conclusion}

We conduct an analysis of the BO + neural predictor framework,
which has recently emerged as a high-performance framework for NAS.
We test several methods for each main component: the encoding, neural predictor,
calibration method, acquisition function, and acquisition
optimization strategy.
We also propose a novel path-based encoding scheme, which improves the performance
of neural predictors.
We use all of this analysis to develop BANANAS, an instantiation
of the BO + neural predictor framework which achieves state-of-the-art performance
on popular NAS search spaces.
Interesting follow-up ideas are to develop multi-fidelity or
successive halving versions of BANANAS. 
Incorporating these approaches with BANANAS could result in a 
significant decrease in the runtime without sacrificing accuracy.

\section*{Acknowledgments}
We thank Jeff Schneider, Naveen Sundar Govindarajulu, and Liam Li for their help with this project.

\newpage

\bibliography{main}

\newpage

\appendix

\section{Related Work Continued}
\label{app:relatedwork}

\paragraph{Bayesian optimization.}
Bayesian optimization is a leading technique for zeroth order optimization
when function queries are expensive \cite{gpml, frazier2018tutorial},
and it has seen great success in hyperparameter optimization for deep
learning \cite{gpml, vizier, hyperband}. The majority of Bayesian optimization
literature has focused on Euclidean or categorical input domains, and has used a
GP model \cite{gpml, vizier, frazier2018tutorial, snoek2012practical}.
There are techniques for parallelizing Bayesian optimization 
\cite{gonzalez2016batch, kandasamy2018parallelised, ovcenavsek2000parallel}.

There is also prior work on using neural network models in
Bayesian optimization for hyperparameter optimization
\cite{snoek2015scalable, springenberg2016bayesian}.
The goal of these papers is to improve the efficiency of 
Gaussian Process-based Bayesian optimization from cubic to linear time,
not to develop a different type of prediction model in order to
improve the performance of BO with respect to the number of iterations. 
In our work, we present techniques which deviate
from Gaussian Process-based Bayesian optimization and see a 
performance boost with respect to the number of iterations.

\paragraph{Neural architecture search.}
Neural architecture search has been studied since at least the 1990s \cite{floreano2008neuroevolution, kitano1990designing, stanley2002evolving},
but the field was revitalized in 2017~\cite{zoph2017neural}. 
Some of the most popular techniques for NAS include evolutionary algorithms \cite{amoebanet, maziarz2018evolutionary}, reinforcement learning \cite{zoph2017neural, enas, pnas, efficientnets, wang2019sample}, Bayesian optimization \cite{nasbot, auto-keras, bayesnas}, gradient descent \cite{darts, darts+, prdarts}, tree search~\cite{alphax,wang2019sample}, 
and neural predictors~\cite{shi2019multi, wen2019neural}.
For a survey on NAS, see \cite{nas-survey}.

Recent papers have called for fair and reproducible experiments~\cite{randomnas,  nasbench}.
In this vein, the NASBench-101~\cite{nasbench}, -201~\cite{nasbench201},
and -301~\cite{nasbench301} datasets 
were created, which contain tens of thousands of pretrained neural architectures.

Initial BO approaches for NAS defined a distance function between 
architectures~\cite{nasbot, auto-keras}.
A few recent papers have used Bayesian optimization with a graph neural
network as a predictor~\cite{ma2019deep, shi2019multi},
however, they do not conduct an ablation study of all components of the 
framework.
In this work, we do not claim to invent the BO + neural predictor framework,
however, we give the most in-depth analysis that we are aware of, which we
use to design a high-performance instantiation of this framework.

\paragraph{Predicting neural network accuracy.}
There are several approaches for predicting the validation accuracy of 
neural networks, such as a layer-wise encoding of neural networks with an 
LSTM algorithm~\cite{peephole}, and a layer-wise encoding and dataset 
features to predict the accuracy for neural network and dataset pairs~\cite{tapas}. 
There is also work in predicting the learning curve of neural networks for hyperparameter 
optimization~\cite{klein2016learning, domhan2015speeding} 
or NAS~\cite{baker2017accelerating} using Bayesian techniques. 
None of these methods have predicted the accuracy of neural networks 
drawn from a cell-based DAG search space such as NASBench or the DARTS 
search space.
Another recent work uses a hypernetwork for neural network prediction 
in NAS~\cite{zhang2018graph}.
Other recent works for predicting neural network accuracy include AlphaX~\cite{alphax},
and three papers which use GCN's to predict neural network 
accuracy~\cite{ma2019deep, shi2019multi, wen2019neural}.

Ensembling of neural networks is a popular approach for uncertainty
estimates, shown in many settings to be more effective than all other methods such as
Bayesian neural networks even for an ensemble of size five~\cite{lakshminarayanan2017simple, beluch2018power, choi2016ensemble, snoek2019can}.

\paragraph{Subsequent work.}
Since its release, several papers have independently shown that BANANAS is a
competitive algorithm for NAS~\citep{remaade, nasbench301,nguyen2020optimal, nasbowl, npenas}.
For example, one paper shows that BANANAS outperforms other algorithms
on NASBench-101
when given a budget of 3200 evaluations~\citep{remaade},
and one paper shows that BANANAS outperforms many popular NAS 
algorithms on NASBench-301~\citep{nasbench301}.
Finally, a recent paper conducted a study on several encodings used for 
NAS~\cite{white2020study}, concluding that neural predictors perform well
with the path encoding, and also improved upon the theoretical results presented in
Section~\ref{sec:methodology}.

\section{Preliminaries Continued}
\label{app:preliminaries}
We give background information on three 
key ingredients of NAS algorithms.
\paragraph{Search space.}
Before deploying a NAS algorithm, we must define the space of neural networks that the algorithm can search through.
Perhaps the most common type of search space for NAS is a \emph{cell-based search space} \cite{zoph2017neural, enas, darts, randomnas, sciuto2019evaluating, nasbench}.
A \emph{cell} consists of a
relatively small section of a neural network, usually 6-12
nodes forming a directed acyclic graph (DAG). 
A neural architecture is then built by
repeatedly stacking one or two different cells on top of each
other sequentially, possibly separated by specialized layers. The
layout of cells and specialized layers is called a
\emph{hyper-architecture}, and this is fixed, while the NAS
algorithm searches for the best cells. The search space over
cells consists of all possible DAGs of a certain size, 
where each node can be one of several operations such as $1\times 1$ convolution, $3\times 3$ convolution, or $3\times 3$ max pooling. 
It is also common to set a restriction on the number of total edges or
the in-degree of each node \cite{nasbench, darts}.
In this work, we focus on NAS over convolutional cell-based search spaces,
though our method can be applied more broadly.

\paragraph{Search strategy.}
The search strategy is the optimization method that the algorithm uses to find the optimal or near-optimal neural architecture from the search space.
There are many varied search strategies, such as Bayesian optimization, evolutionary search, reinforcement learning, and gradient descent. 
In Section \ref{sec:methodology}, we introduced the search strategy we study in this
paper: Bayesian optimization with a neural predictor.

\paragraph{Evaluation method.}
Many types of NAS algorithms consist of an iterative framework in which the algorithm
chooses a neural network to train, computes its validation error, and uses this result
to guide the choice of neural network in the next iteration. The simplest instantiation
of this approach is to train each neural network in a fixed way, i.e., the algorithm
has black-box access to a function that trains a neural network for $x$ epochs and then
returns the validation error.
Algorithms with black-box
evaluation methods can be compared by returning the architecture with the lowest 
validation error after a certain number of queries to the black-box function. There
are also multi-fidelity methods, for example, when a NAS algorithm chooses the number of
training epochs in addition to the architecture.

\section{Path Encoding Theory}\label{app:methodology}

In this section, we give the full details of Theorem~\ref{thm:path_length_informal}
from Section~\ref{sec:methodology}, which shows that with high probability, 
the path encoding can be truncated significantly without losing information.

Recall that the size of the path encoding is equal to the
number of unique paths, which is $\sum_{i=0^n}r^i$, where
$n$ is the number of nodes in the cell, and $r$ is the
number of operations to choose from at each node.
This is at least $r^n$.
By contrast, the adjacency matrix encoding scales quadratically
in $n$.

However, the vast majority of the paths rarely show up in any neural architecture 
throughout a full run of a NAS algorithm.
This is because many NAS algorithms can only sample architectures from a random
procedure or mutate architectures drawn from the random procedure.
Now we will give the full details of Theorem~\ref{thm:path_length_informal},
showing that the vast majority of paths
have a very low probability of occurring in a cell outputted from \texttt{random\_spec()}, 
a popular random procedure used by many NAS algorithms~\cite{nasbench},
including BANANAS.
Our results show that by simply truncating the least-likely paths,
the path encoding scales \emph{linearly} in the size of the cell,
with an arbitrarily small amount of information loss.
We back this up with experimental evidence in 
Figures~\ref{fig:path_encoding} and~\ref{fig:extras}, and 
Table~\ref{tab:nasbench-probs}.

We start by defining \texttt{random\_spec()}, the procedure to output a random
neural architecture.

\begin{definition}
\label{def:random_graph}
Given integers $n, r$, and $k<\nicefrac{n(n-1)}{2}$,
a random graph $G_{n, k, r}$ is generated as follows:
\emph{(1)} Denote $n$ nodes by 1 to $n$. 
\emph{(2)} Label each node randomly with one of $r$ operations.
\emph{(3)} For all $i<j$, add edge $(i,j)$ with probability $\frac{2k}{n(n-1)}$.
\emph{(4)} If there is no path from nodes 1 to $n$, \texttt{goto} \texttt{(1)}.
\end{definition}

The probability value in step (3) is chosen so that the expected
number of edges after this step is exactly $k$.
Recall that we use `path' to mean a path from node 1
to node $n$.
We restate the theorem formally.
Denote $\mathcal{P}$ as the set of all possible paths from 
node 1 to node $n$ that could occur in $G_{n,k,r}$.

\noindent\textbf{Theorem~\ref{thm:path_length_informal}~(formal).}
\emph{
Given integers $r, c>0$, there exists $N$ such that for all
$n>N$, there exists a set of $n$ paths  $\mathcal{P}'\subseteq \mathcal{P}$
such that
\begin{equation*}
P(\exists p\in G_{n,n+c,r}\cap\mathcal{P}\setminus\mathcal{P}')\leq \frac{1}{n^{2}}.
\end{equation*}
}

This theorem says that when $k=n+c$, and when $n$ is large enough
compared to $c$ and $r$,
then we can truncate the path encoding to a set 
$\mathcal{P}'$ of size $n$, 
because the probability that \texttt{random\_spec()} 
outputs a graph $G_{n, k, r}$ 
with a path outside of $\mathcal{P}'$ is very small.

Note that there are two caveats to this theorem.
First, BANANAS may mutate architectures drawn from 
Definition~\ref{def:random_graph}, and Theorem~\ref{thm:path_length_informal}
does not show the probability of paths from mutated architectures is small.
However, our experiments (Figures~\ref{fig:path_encoding} and~\ref{fig:extras})
give evidence that
the mutated architectures do not change the distribution of paths too much.
Second, the most common paths in Definition~\ref{def:random_graph} are not necessarily
the paths whose existence or non-existence give the most entropy in predicting the validation
accuracy of a neural architecture. Again, while this is technically true, our experiments
back up Theorem~\ref{thm:path_length_informal} as a reasonable argument
that truncating the path encoding does not sacrifice performance.

Denote by $G'_{n, k, r}$ the random graph outputted by
Definiton~\ref{def:random_graph} without step (4).
In other words, $G'_{n, k, r}$ is a random graph that
could have no path from node 1 to node $n$.
Since there are $\frac{n(n-1)}{2}$ pairs $(i,j)$ such that
$i<j$, the expected number of edges of $G'_{n, k, r}$ is $k$.
For reference, in the NASBench-101 dataset, there are $n=7$ nodes
and $r=3$ operations, and the maximum number of edges is 9.

We choose $\mathcal{P}'$ as the $n$ shortest paths
from node 1 to node $n$.
The argument for Theorem~\ref{thm:path_length_informal}
relies on a simple concept:
the probability that $G_{n,k,r}$ contains a
long path (length $>\log_r n$) 
is much lower than the probability that it contains a short path.
For example, the probability that
$G'_{n,k,r}$ contains a path of length $n-1$
is very low,
because there are $\Theta(n^2)$ potential edges 
but the expected number of edges is $n+O(1)$.
We start by upper bounding the length of the $n$ shortest paths.

\begin{lemma} \label{lem:path_length}
Given a graph with $n$ nodes and $r$ node labels, 
there are fewer than $n$ paths of length 
less than or equal to $\log_r n - 1$.
\end{lemma}

\begin{proof}[\textbf{Proof.}]
The number of paths of length $\ell$ is $r^\ell$, since there
are $r$ choices of labels for each node.
Then
\begin{equation*}
1+r+\cdots+r^{\lceil\log_r n\rceil - 1}
=\frac{ r^{\lceil\log_r n\rceil}-1}{r-1}=\frac{n-1}{r-1}<n.
\end{equation*}
\end{proof}

To continue our argument, we will need the following well-known
bounds on binomial coefficients, e.g.\ \cite{stanica2001good}.

\begin{theorem} \label{thm:binomial}
Given $0\leq \ell\leq n$, we have
\begin{equation*}
    \left(\frac{n}{\ell}\right)^\ell \leq \binom{n}{\ell}
    \leq \left(\frac{en}{\ell}\right)^\ell.
\end{equation*}
\end{theorem}

Now we define $a_{n,k,\ell}$ as the expected number of paths
from node 1 to node $n$ of length $\ell$ in $G'_{n,k,r}$.
Formally,
\begin{equation*}
a_{n,k,\ell}=\mathbb{E}\left[\left|p\in\mathcal{P}\right|\mid |p|=\ell\right].
\end{equation*}
The following lemma, which is the driving force behind
Theorem~\ref{thm:path_length_informal}, shows that the value of
$a_{n,k,\ell}$ for small $\ell$ is much larger than 
the value of $a_{n,k,\ell}$ for large $\ell$.

\begin{lemma} \label{lem:a_nk}
Given integers $r, c>0$, 
then there exists $n$ such that for $k=n+c$, we have
\begin{equation*}
\sum_{\ell=\log_r n}^{n-1} a_{n,k,\ell} < \frac{1}{n^{3}}\text{  and  } 
a_{n,k,1}>\frac{1}{n}.
\end{equation*}
\end{lemma}

\begin{proof}[\textbf{Proof.}]
We have that 
\begin{equation*}
a_{n,k,\ell}=\binom{n-2}{\ell-1}\left(\frac{2k}{n(n-1)}\right)^\ell.
\end{equation*}
This is because on a path from node 1 to $n$ of length $\ell$,
there are $\binom{n-2}{\ell-1}$ choices of intermediate
nodes from 1 to $n$. Once the nodes are chosen, we need all $\ell$
edges between the nodes to exist, and each edge exists independently 
with probability $\frac{2}{n(n-1)}\cdot k.$

When $\ell=1$, we have $\binom{n-2}{\ell-1}=1$.
Therefore,
\begin{equation*}
a_{n,k,1}=\left(\frac{2k}{n(n-1)}\right)\geq \frac{1}{n},
\end{equation*}
for sufficiently large $n$.
Now we will derive an upper bound for $a_{n,k,\ell}$ using
Theorem~\ref{thm:binomial}.

\begin{align*}
a_{n,k,\ell} &= \binom{n-2}{\ell-1}\left(\frac{2k}{n(n-1)}\right)^\ell \\
&\leq \left(\frac{e(n-2)}{\ell-1}\right)^{\ell-1} \left(\frac{2k}{n(n-1)}\right)^\ell \\
&\leq \left(\frac{2k}{n(n-1)}\right) \left(\frac{2ek(n-2)}{(\ell-1)n(n-1)}\right)^{\ell-1} \\
&\leq \left(\frac{4}{n}\right)\left(\frac{4e}{\ell-1}\right)^{\ell-1}
\end{align*}
The last inequality is true because
$k/(n-1)=(n+c)/(n-1)\leq 2$ for sufficiently large $n$.
Now we have

\begin{align}
\sum_{\ell=\log_r n}^{n-1} a_{n,k,\ell}
&\leq \sum_{\ell=\log_r n}^{n-1} \left(\frac{4}{n}\right)\left(\frac{4e}{\ell-1}\right)^{\ell-1}\nonumber \\
&\leq \sum_{\ell=\log_r n}^{n-1}
\left(\frac{4e}{\ell-1}\right)^{\ell-1} \nonumber\\
&\leq \sum_{\ell=\log_r n}^{n-1}
\left(\frac{4e}{\log_r n}\right)^{\ell-1} \nonumber\\
&\leq \left(\frac{4e}{\log_r n}\right)^{\log_r n}
\sum_{\ell=0}^{n-\log_r n}
\left(\frac{4e}{\log_r n}\right)^\ell \nonumber\\
&\leq \left(e\right)^{3\log_r n}\left(\frac{1}{\log_r n}\right)^{\log_r n}\cdot 2\label{eq:sum}\\
&\leq 2\left(n\right)^3\left(\frac{1}{n}\right)^{\log_r\log_r n}\label{eq:loglog}\\
&\leq \left(\frac{1}{n}\right)^{\log_r\log_r n-4}\nonumber\\
&\leq \frac{1}{n^{3}}.\nonumber
\end{align}

In inequality~\ref{eq:sum}, we use the fact that for large enough $n$,
$\frac{4e}{\log_r n}<\frac{1}{2}$, therefore, 
\begin{equation*}
\sum_{\ell=0}^{n-\log_r n}\left(\frac{4e}{\log_r n}\right)^\ell
\leq \sum_{\ell=0}^{n-\log_r n}\left(\frac{1}{2}\right)^\ell\leq 2
\end{equation*}

In inequality~\ref{eq:loglog}, we use the fact that
\begin{align*}
\left(\log n\right)^{\log n}&=\left(e^{\log \log n}\right)^{\log n}
=\left(e^{\log n}\right)^{\log \log n}\\
&=n^{\log \log n} \qedhere
\end{align*}
\end{proof}

Now we can prove Theorem~\ref{thm:path_length_informal}.

\begin{proof}[\textbf{Proof of Theorem~\ref{thm:path_length_informal}}]

Recall that $\mathcal{P}$ denotes the set of all 
possible paths from node 1 to node $n$ that could be present
in $G_{n,k,r}$,
and let $\mathcal{P}'=\{p\mid |p|<\log_r n - 1\}$.
Then by Lemma~\ref{lem:path_length}, $|\mathcal{P}|<n$.
In Definition~\ref{def:random_graph}, the probability that we return a
graph in step~(4) is at least the probability that there exists an edge
from node 1 to node $n$. This probability is $\geq\frac{1}{n}$ from
Lemma~\ref{lem:a_nk}.
Now we will compute the probability that there exists a path in
$\mathcal{P}\setminus \mathcal{P}'$ in $G_{n,k,r}$
by conditioning on returning a graph in step~(4).
The penultimate inequality is due to Lemma~\ref{lem:a_nk}.

\begin{align*}
&P(\exists p\in G_{n,k,r}\cap \mathcal{P}\setminus\mathcal{P}')\\
&= P(\exists p\in G'_{n,k,r}\cap\mathcal{P}\setminus\mathcal{P}'
\mid \exists q\in G'_{n,k,r}\cap\mathcal{P})\\
&= \frac{P(\exists p\in G'_{n,k,r}\cap\mathcal{P}\setminus\mathcal{P}')}
{P(\exists q\in G'_{n,k,r}\cap\mathcal{P})}\\
&\leq \left(\frac{1}{n^{3}}\right) / \left(\frac{1}{n}\right)
\leq \frac{1}{n^{2}} \qedhere
\end{align*}
\end{proof}
\section{Additional Experiments and Details}
\label{app:experiments}

In this section, we present details and supplementary experiments from 
Sections~\ref{sec:methodology} and~\ref{sec:experiments}.
In the first subsection, we give a short description and implementation
details for all 15 of the NAS algorithms we tested in Section~\ref{sec:experiments},
as well as additional details from Section~\ref{sec:experiments}.
Next, we give an exhaustive experiment on the
BO + neural predictor framework.
After that, we evaluate BANANAS
on the three datasets in NASBench-201.
Then,
we discuss the NASBench-101 API and conduct additional experiments.
Finally,
we study the effect of the length of the path encoding on the performance of BANANAS
on NASBench-201.


\subsection{Details from Section~\ref{sec:experiments}} \label{app:experiments:details}
Here, we give more details on the NAS algorithms we compared
in Section~\ref{sec:experiments}.

\textbf{Regularized evolution.} 
This algorithm consists of iteratively mutating the best achitectures
out of a sample of all architectures evaluated so 
far~\cite{real2019regularized}.
We used the~\cite{nasbench} implementation although we
changed the population size from 50 to 30 to account for fewer total queries.
We also found that in each round, removing the architecture with the worst validation
accuracy performs better than removing the oldest architecture, so this is the
algorithm we compare to. (Technically this would make the algorithm 
closer to standard evolution.)

\textbf{Local search.} 
Another simple baseline, local search iteratively evaluates all architectures
in the neighborhood of the architecture with the lowest validation error found so
far. For NASBench-101, the ``neighborhood'' means all architectures which differ
from the current architecture by one operation or one edge.
We used the implementation from White et al.~\citet{white2020local},
who showed that local search is a state-of-the-art
approach on NASBench-101 and NASBench-201.

\textbf{Bayesian optimization with a GP model.}
We set up Bayesian optimization with a Gaussian process model and UCB acquisition.
In the Gaussian process, we set the distance function between two neural networks
as the sum of the Hamming distances between the adjacency matrices and the list of operations.
We use the ProBO implementation~\cite{neiswanger2019probo}.

\textbf{NASBOT.}
Neural architecture search with Bayesian optimization 
and optimal transport (NASBOT)~\cite{nasbot} 
works by defining a distance function
between neural networks by computing the similarities between layers
and then running an optimal transport algorithm to find the minimum
earth-mover's distance between the two architectures.
Then Bayesian optimization is run using this distance function.
The NASBOT algorithm is specific to macro NAS, and we put in 
a good-faith effort to implement it in the cell-based setting.
Specifically, we compute the distance between two cells by taking
the earth-mover's distance between the set of row-sums, column-sums,
and node operations.
This is a version of the OTMANN distance~\cite{nasbot}, 
defined for the cell-based setting.

\textbf{Random search.} 
The simplest baseline, random search, draws $n$ architectures at
random and outputs the architecture with the lowest validation error.
Despite its simplicity, multiple papers have concluded that random search
is a competitive baseline for NAS algorithms \cite{randomnas, sciuto2019evaluating}.
In Table~\ref{table:darts}, we also compared to Random Search with Weight-Sharing,
which uses shared weights to quickly compare orders of magnitude more architectures. 

\textbf{AlphaX.}
AlphaX casts NAS as a reinforcement learning problem, using a neural network to
guide the search~\cite{alphax}.
Each iteration, a neural network is trained to select the best action,
such as making a small change to, or growing, the current architecture.
We used the open-source implementation of AlphaX as is~\cite{alphax}.

\textbf{BOHAMIANN.}
Bayesian Optimization with Hamiltonian Monte Carlo Artificial Neural Networks 
(BOHAMIANN)~\cite{springenberg2016bayesian}
is an approach which fits in to the ``BO + neural predictor'' framework.
It uses a Bayesian neural network (implemented using Hamiltonian Monte Carlo) as the
neural predictor. We used the BOHAMIANN implementation of the Bayesian neural 
network~\cite{springenberg2016bayesian} with our own outer BO wrapper, so that
we could accurately compare different neural predictors within the framework.

\textbf{REINFORCE.} 
We use the NASBench-101 implementation of REINFORCE~\cite{reinforce}.
Note that this was the best reinforcement learning-based NAS algorithm released by
NASBench-101, outperforming other popular approaches such as a
1-layer LSTM controller trained with PPO~\cite{nasbench}.

\textbf{GCN Predictor.}
We implemented a GCN predictor~\cite{wen2019neural}. Although the code is not open-sourced,
we found an open-source implementation online~\cite{zhang2020neural}.
We used this implementation, keeping the hyperparameters the same as in the original
paper~\cite{wen2019neural}.

\textbf{BONAS.}
We implemented BONAS~\cite{shi2019multi}. Again, the code was not open-sourced,
so we used the same GCN implementation as above~\cite{zhang2020neural},
using our own code for the outer BO wrapper.

\textbf{DNGO.}
Deep Networks for Global Optimization (DNGO) is an implementation of Bayesian optimization
using adaptive basis regression using neural networks instead of
Gaussian processes to avoid the cubic scaling.
We used the open-source code~\cite{snoek2015scalable}.

\textbf{BOHB.}
Bayesian Optimization HyperBand (BOHB) combines multi-fidelity
Bayesian optimization with principled early-stopping from Hyperband~\cite{bohb}.
We use the NASBench implementation~\cite{nasbench}.

\textbf{TPE.}
Tree-structured Parzen estimator (TPE) is a BO-based
hyperparameter optimization algorithm based on
adaptive Parzen windows.
We use the NASBench implementation~\cite{nasbench}.

\textbf{DARTS.}
DARTS~\cite{darts} is a popular first-order (sometimes called ``one-shot'') NAS
algorithm. In DARTS, the neural network parameters and the architecture hyperparameters
are optimized simultaneously using alternating steps of gradient descent.
In Table~\ref{table:darts}, we reported the published numbers from the paper, and
then we retrained the architecture published by the DARTS paper, five times, 
to account for differences in hardware.

\textbf{ASHA.}
Asyncrhonous Successive Halving Algorithm (ASHA) is an algorithm that
uses asynchronous parallelization and early-stopping. 
As with DARTS, we reported both the published number and the numbers we achieved by 
retraining the published architecture on our hardware.

\textbf{Additional notes from Section~\ref{sec:experiments}.}
We give the results from Figure~\ref{fig:nasbench_main} (right) into a table (Table~\ref{table:full}).

\begin{table*}[t]
\caption{Comparison of the architectures with the lowest test error (averaged over 200 trials)
returned by NAS algorithms after 150 architecture evaluations on NASBench-101.}
\setlength\tabcolsep{0pt}
\begin{tabular*}{\textwidth}{ @{\extracolsep{\fill}}  l c c c} 
\toprule
\multicolumn{1}{c}{NAS Algorithm} & \multicolumn{1}{c}{Source} & \multicolumn{1}{c}{Method} & \multicolumn{1}{c}{Test Error} \\
\midrule
REINFORCE & \cite{reinforce} & Reinforcement learning & 6.436 \\
TPE & \cite{tpe} & BO (Parzen windows) & 6.415 \\
BOHB & \cite{bohb} & BO (successive halving) & 6.356  \\
Random search & \cite{randomnas} & Random search & 6.341\\
GCN Pred.\ & \cite{wen2019neural} & GCN & 6.331 \\
BO w.\ GP & \cite{snoek2012practical} & BO (Gaussian process) & 6.267  \\
NASBOT & \cite{nasbot} & BO (Gaussian process) & 6.250  \\
AlphaX & \cite{alphax} & Monte Carlo tree search & 6.233  \\
Reg.\ Evolution & \cite{real2019regularized} & Evolution & 6.109 \\
DNGO & \cite{snoek2015scalable} & BO (neural networks) & 6.085  \\
BOHAMIANN & \cite{springenberg2016bayesian} & BO (Bayesian NN) & 6.010  \\
BONAS & \cite{shi2019multi} & BO (GCN) & 5.954  \\
Local search & \cite{white2020local} & Local search & 5.932  \\
BANANAS & \hspace{-0.5mm}Ours & BO (path encoding) & \textbf{5.923} \\
\bottomrule
\end{tabular*} 
\label{table:full}
\end{table*}

In the main NASBench-101 experiments, Figure~\ref{fig:nasbench_main},
we added an isomorphism-removing subroutine to any algorithm that uses the adjacency
matrix encoding. This is because multiple adjacency matrices can map to the same architecture.
With the path encoding, this is not necessary.
Note that without the isomorphism-removing subroutine, 
algorithms using the adjacency matrix encoding may perform significantly worse (e.g.,
we found this to be true for BANANAS with the adjacency matrix encoding).
This is another strength of the path encoding.

In Section~\ref{sec:experiments}, we described the details of running
BANANAS on the DARTS search space, which resulted in an architecture.
We show this architecture in Figure~\ref{fig:bananas_arch}.

\begin{figure*}
  \centering
    \includegraphics[width=0.43\textwidth]{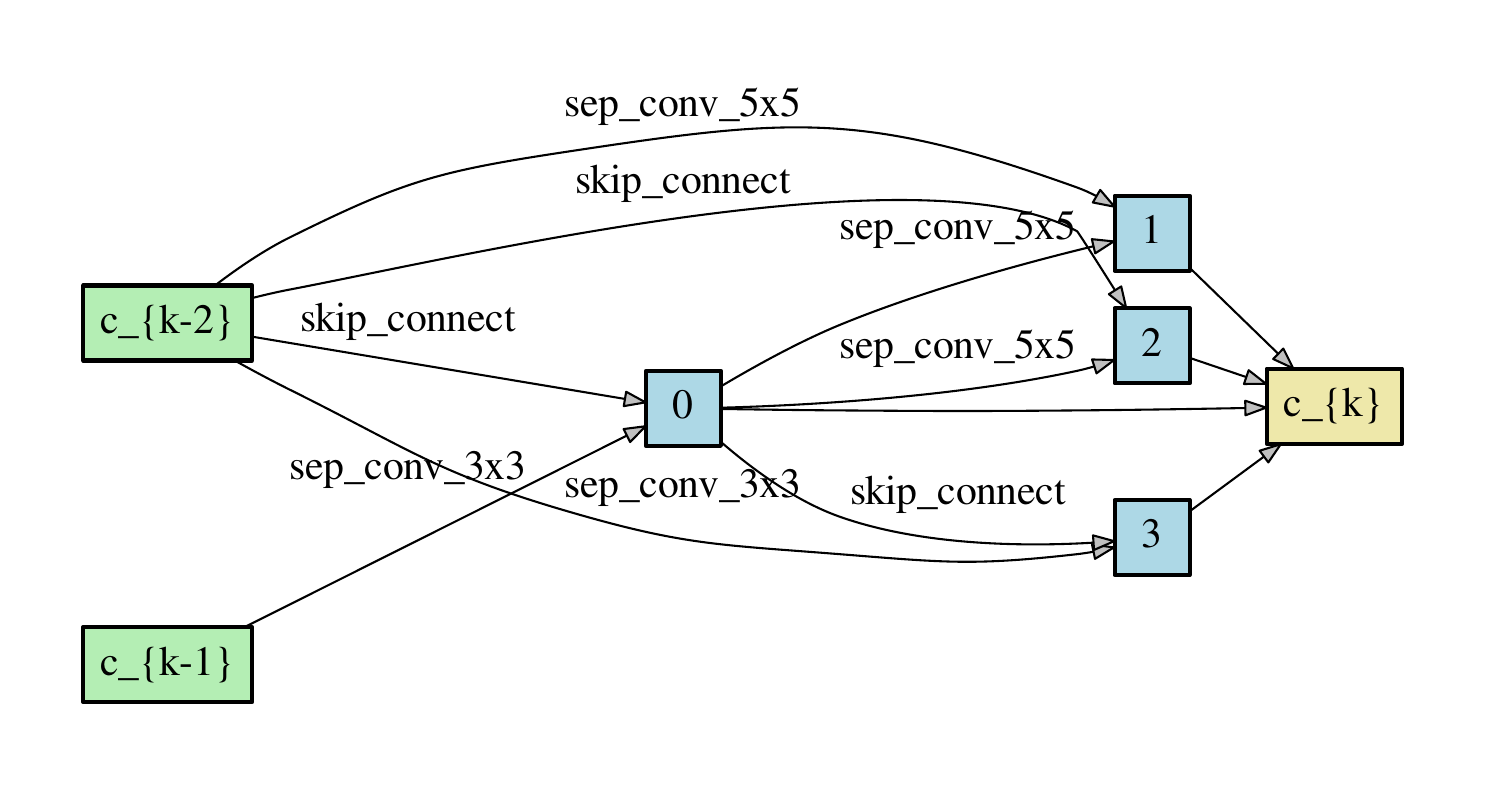}
  \includegraphics[width=0.5\textwidth]{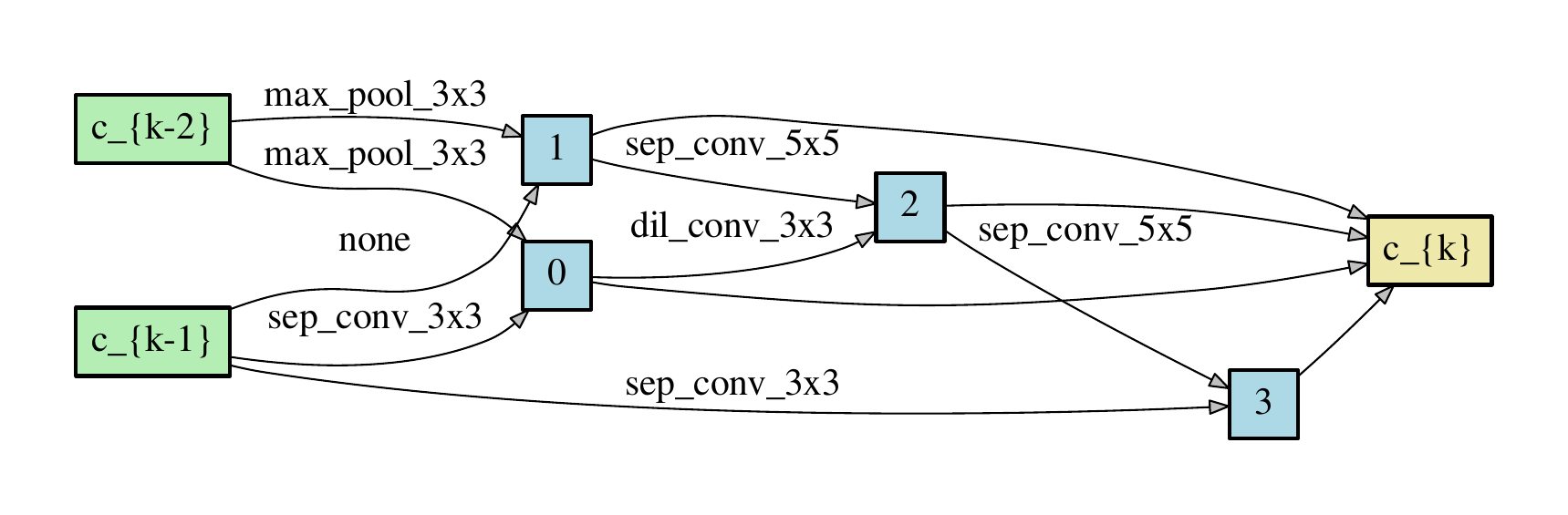}
\caption{The best neural architecture found by BANANAS in the DARTS space. Normal cell (left) and reduction cell (right).}
  \label{fig:bananas_arch}
\end{figure*}

\subsection{Exhaustive Framework Experiment}\label{app:experiments:exhaustive}
In Section~\ref{sec:methodology}, we conducted experiments on each component
individually while keeping all other components fixed.
However, this experimental setup implicitly assumes that all components are linear with
respect to performance. For example, we showed GCN performs worse than 
the path encoding with ITS, and UCB performs worse than ITS using the path encoding, but we never tested GCN together with UCB -- what if it outperforms ITS with the path encoding?

In this section we run a more exhaustive experiment by testing the 18 most promising
configurations. 
We take all combinations of the highest-performing components
from Section~\ref{sec:methodology}).
Specifically, we test all combinations of \{UCB, EI, ITS\},
\{mutation, mutation+random\}, and 
\{GCN, path enc., trunc.\ path enc.)\}
as acquisition function, acquisition optimization strategy, and neural predictor.
We use the same experimental setup as in Section~\ref{sec:methodology}, and we run
500 trials of each algorithm.
See Figure~\ref{fig:framework_study}.
The overall best-performing algorithm was Path-ITS-Mutation, which was the same conclusion
reached in Section~\ref{sec:methodology}.
The next best combinations are Path-ITS-Mut+Rand and Trunc-ITS-Mut+Rand.
Note that there is often very little difference between the path encoding and truncated path encoding,
all else being equal.
The results show that each component has a fairly linear relationship with respect to performance:
mutation outperforms mutation+random; ITS outperforms UCB which outperforms EI;
and both the path and truncated path encodings outperform GCN.

\begin{figure*}
  \centering
    \includegraphics[width=0.48\textwidth]{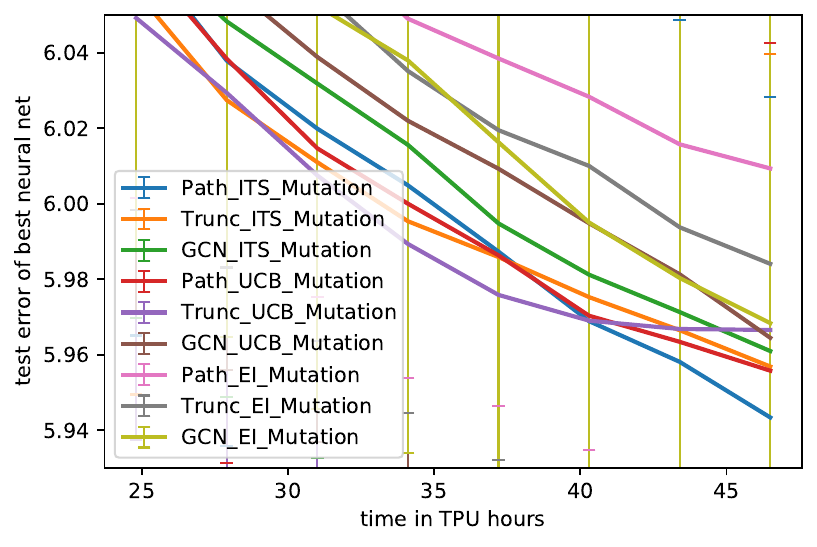}
  \includegraphics[width=0.48\textwidth]{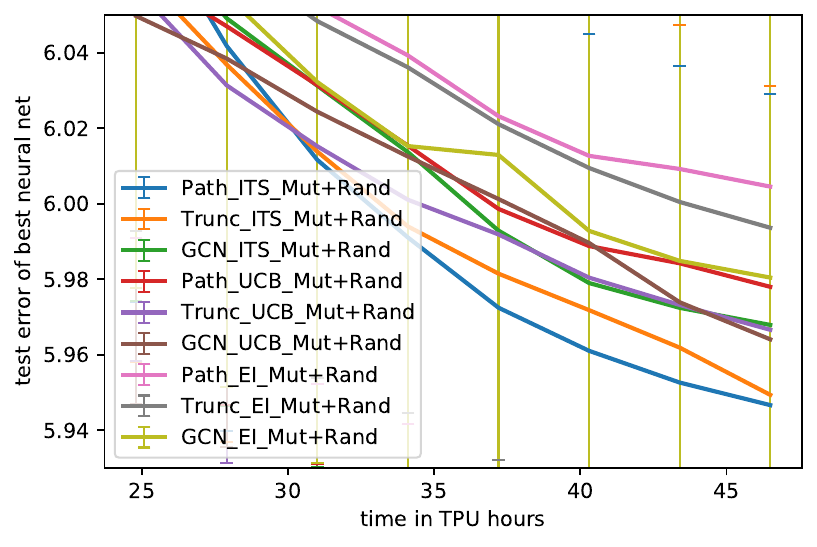}
\caption{A more exhaustive study of the different components in the 
BO + neural predictor framework.}
  \label{fig:framework_study}
\end{figure*}

\subsection{Results on NASBench-201}\label{app:experiments:201}
We described the NASBench-201 dataset in Section~\ref{sec:experiments}.
The NASBench-201 dataset is similar to NASBench-101. 
Note that NASBench-201 is much smaller even than NASBench-101: it is originally 
size 15625, but it only contains 6466 unique architectures after all isomorphisms
are removed~\citep{nasbench201}. 
By contrast, NASBench-101 has about 423,000 architectures after removing
isomorphisms. Some papers have claimed that NASBench-201 may be too small to effectively
benchmark NAS algorithms~\cite{white2020local}.
However, one upside of NASBench-201 is that it contains three image datasets instead of
just one: CIFAR-10, CIFAR-100, and ImageNet-16-120.

Our experimental setup is the same as for NASBench-101 in Section~\ref{sec:experiments}.
See Figure~\ref{fig:201_baselines}.
As with NASBench-101, at each point in time, we plotted the test error of the architecture
with the best validation error found so far (and then we averaged this over 200 trials).
However, on NASBench-201, the validation and test errors are not as highly correlated
as on NASBench-101, which makes it possible for the NAS algorithms to overfit to the
validation errors. Specifically for ImageNet16-120, the lowest validation error
out of all 15625 architectures is 53.233, and the corresponding test error is
53.8833. However, there are 23 architectures which have a higher validation loss but
lower test loss (and the lowest overall test loss is 53.1556).
Coupled with the small size of NASBench-201, this can cause NAS performance over time
to not be strictly decreasing (see Figure~\ref{fig:201_baselines} bottom left).
Therefore, we focus on the plots of the validation error over time 
(Figure~\ref{fig:201_baselines} top row).

Due to the extremely small size of the search space as described above, several algorithms
tie for the best performance. We see that BANANAS ties for the best performance on
CIFAR-10 and CIFAR-100. On ImageNet16-120, it ties for the best performance after 40 GPU
hours, but NASBOT and BO w.\ GP reach top performance more quickly.
We stress that we did not change any hyperparameters or any other part of the code of BANANAS,
when moving from NASBench-101 to all three NASBench-201 datasets, which shows that
BANANAS does not need to be tuned.

\begin{figure*}
\centering %
\includegraphics[width=0.32\textwidth]{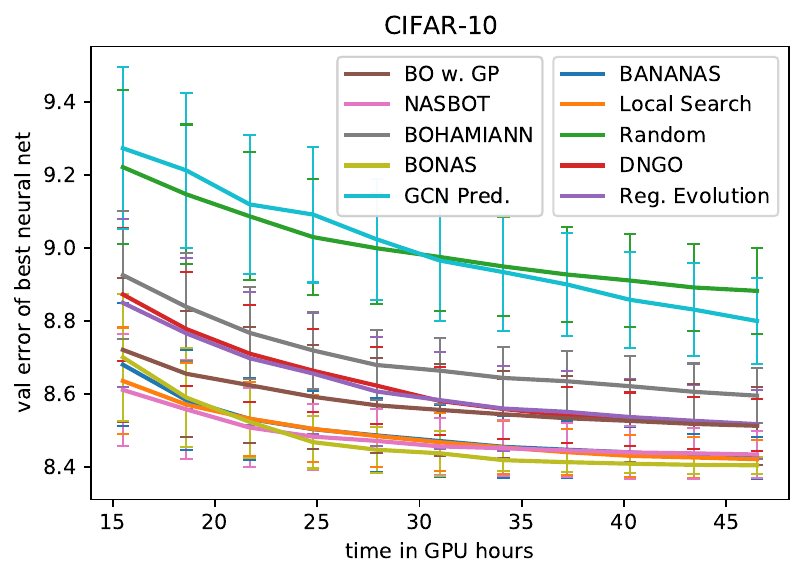}
\hspace{-3pt}
\includegraphics[width=0.32\textwidth]{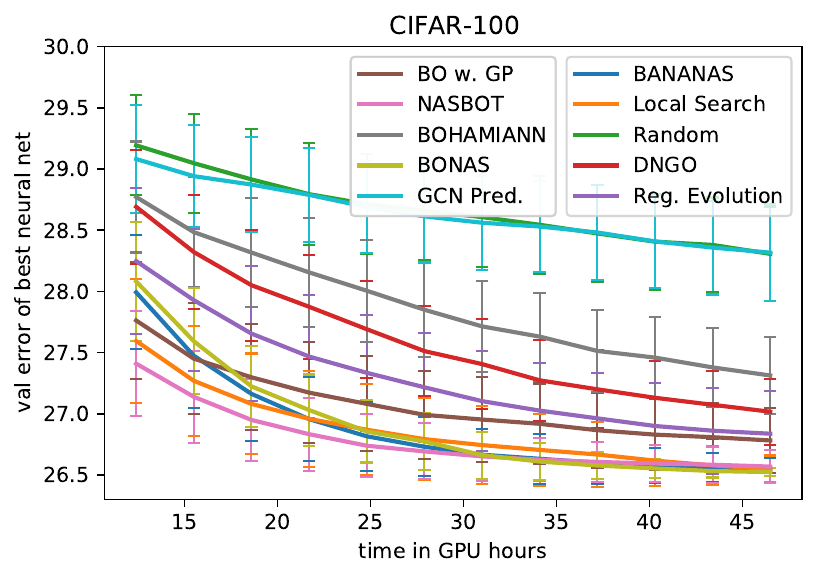}
\hspace{-3pt}
\includegraphics[width=0.32\textwidth]{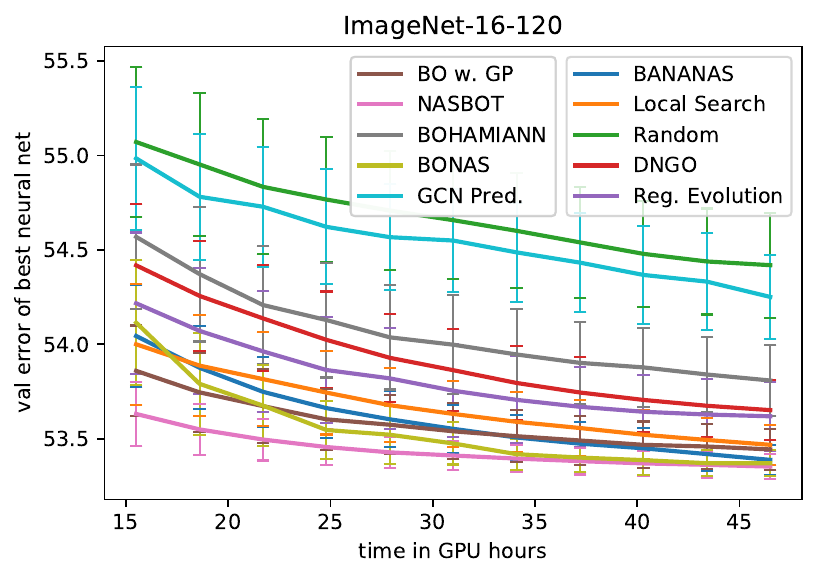}
\includegraphics[width=0.32\textwidth]{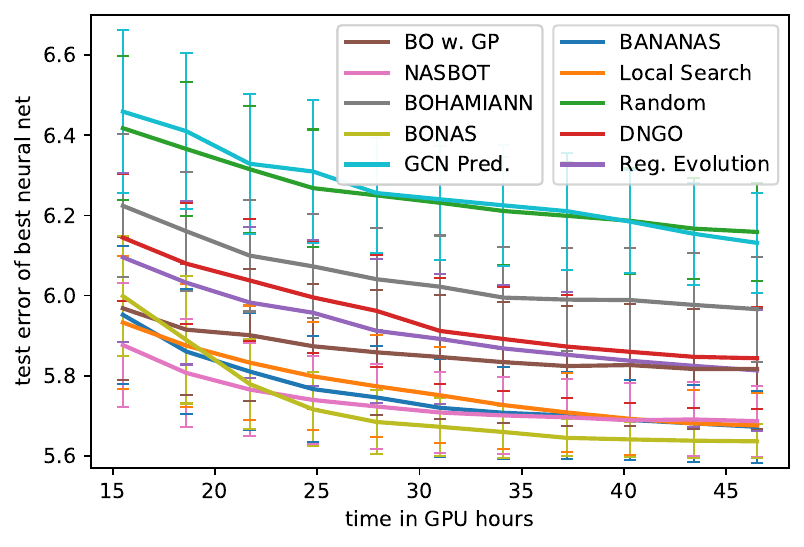}
\hspace{-3pt}
\includegraphics[width=0.32\textwidth]{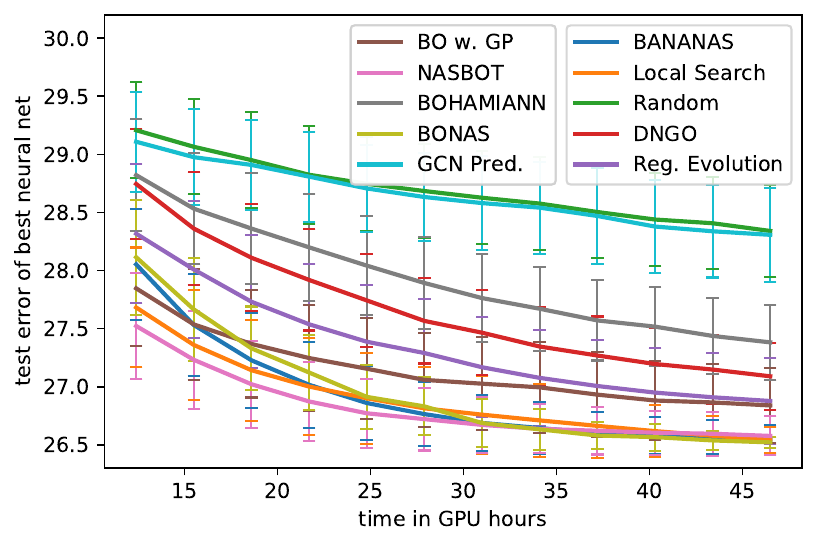}
\hspace{-3pt}
\includegraphics[width=0.32\textwidth]{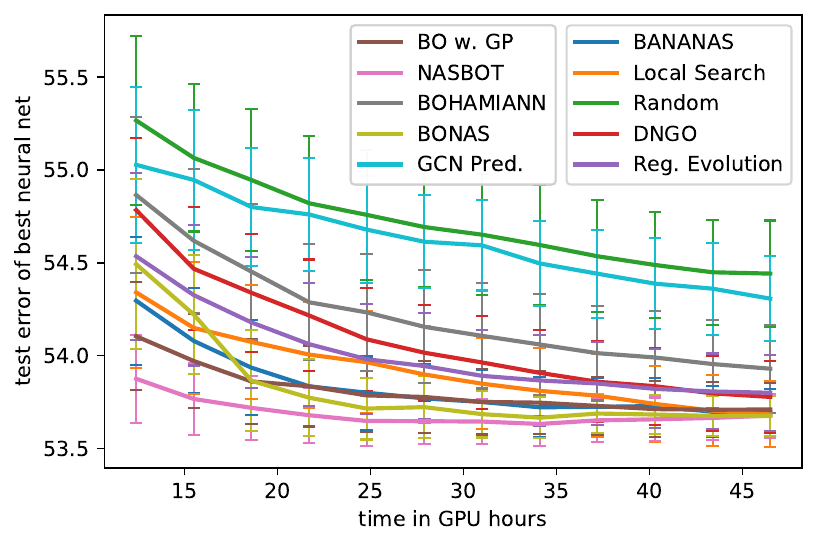}
\caption{
Results on NASBench-201. Top row is validation error, bottom row is test error.
CIFAR-10 (left), CIFAR-100 (middle), and ImageNet-16-120 (right).
}
\label{fig:201_baselines}
\end{figure*}

\subsection{NASBench-101 API}
\label{app:experiments:rand_val_error}
In the NASBench-101 dataset, each architecture was trained to 108 epochs three separate times with different random seeds.
The original paper conducted experiments by 
(\emph{1}) choosing a random validation error when evaluating each architecture, and then reporting the mean test error at the conclusion of the NAS algorithm. 
The most realistic setting is:
(\emph{2}) choosing a random validation error when evaluating each architecture, 
and then reporting the corresponding test error, and an approximation of this is
(\emph{3}) using the mean validation error in the search, and reporting the mean test error at the end.
However, (\emph{2}) is currently not possible with the NASBench-101 API, so our options are
(\emph{1}) or (\emph{3}), neither of which is perfect.
(\emph{3}) does not capture the uncertainty in real NAS experiments, while (\emph{1}) does not give as
accurate results because of the differences between random validation errors and mean test errors.
We used (\emph{3}) for Figure~\ref{fig:nasbench_main}, and now we use
(\emph{1}) in Figure~\ref{fig:extras}.
We found the overall trends to be the same in Figure~\ref{fig:extras} 
(in particular, BANANAS still distinctly outperforms all other algorithms after 40 iterations),
but the Bayesian optimization-based methods performed better at the very start.

\begin{figure*}
\centering %
\includegraphics[width=0.45\textwidth]{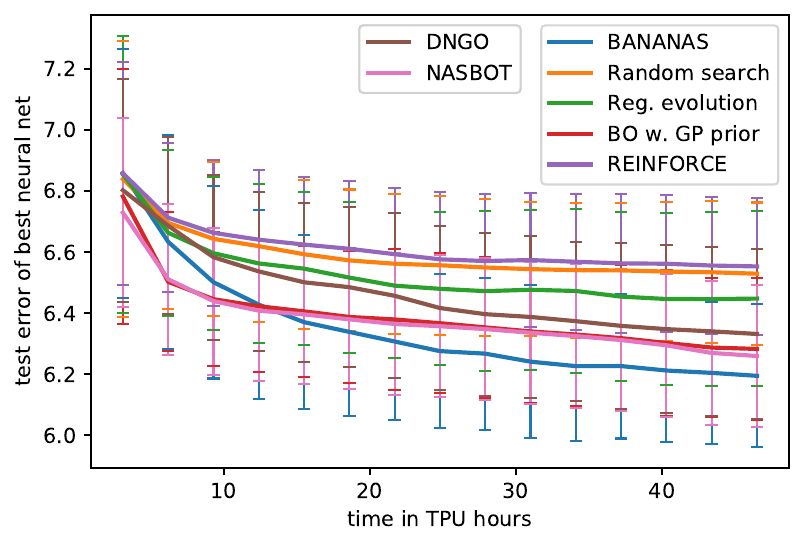}
\hspace{0.5cm}
\includegraphics[width=0.45\textwidth]{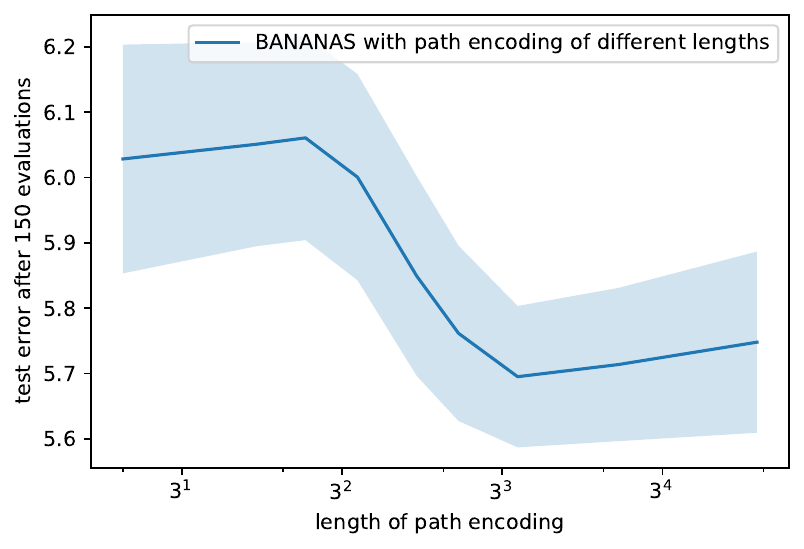}
\caption{
NAS experiments with random validation error (left).
Performance of BANANAS on NASBench-201 with CIFAR-10
with the path encoding truncated to different lengths (right).
}
\label{fig:extras}
\end{figure*}

\subsection{Path encoding length}
\label{app:experiments:path_length}
In Section~\ref{sec:methodology}, we gave theoretical results which suggested
that truncating the path encoding may not decrease performance of NAS algorithms
such as BANANAS. We backed this up by plotting performance of BANANAS vs.\ truncation
length of the path encoding (Figure~\ref{fig:path_encoding}.
Specifically, we ran BANANAS up to 150 evaluations for truncation lengths of
$3^0$, $3^1,\dots,3^5$ and plotted the results.

Now, we conduct the same experiment for NASBench-201.
We run BANANAS up to 150 evaluations on CIFAR-10 on 
NASBench-201 for truncation lengths of 1, 2, 5, 7, 10, 15, 20, 30, 60, and 155
(where 155 is the total number of paths for NASBench-201).
See Figure~\ref{fig:extras}.
We see that truncating from 155 down to just 30 has no
decrease in performance.
In fact, similar to NASBench-101, 
the performance after truncation actually \emph{improves} up to a certain point. 
We believe this is because with the full-length encoding, 
the neural predictor overfits to very rare paths.

Next, we give a table of the probabilities
of paths by length from NASBench-101 generated from
\texttt{random\_spec()} (i.e., Definition~\ref{def:random_graph}).
These probabilities were computed experimentally
by making 100000 calls to \texttt{random\_spec()}.
See Table~\ref{tab:nasbench-probs}.
This table gives further experimental evidence to support
Theorem~\ref{thm:path_length_informal}, because it shows that
the longest paths are exceedingly rare.

\begin{table*}

\centering
\caption{Probabilities of path lengths in NASBench-101 using \texttt{random\_spec()}.}

\begin{minipage}[c]{.85\textwidth}

\setlength\tabcolsep{0pt}
\begin{tabular*}{\textwidth}{l @{\extracolsep{\fill}}*{8}{S[table-format=1.4]}} 
\toprule
\multicolumn{1}{c}{Path Length} & \multicolumn{1}{c}{Probability} & \multicolumn{1}{c}{Total num.\ paths} & \multicolumn{1}{c}{Expected num.\ paths} \\
\midrule
1 & 0.200 & 1 & 0.200 \\
2 & 0.127 & 3 & 0.380 \\
3 & 3.36 $\times$ $10^{-2}$ & 9 & 0.303 \\
4 & 3.92 $\times$ $10^{-3}$ & 27 & 0.106 \\
5 & 1.50 $\times$ $10^{-4}$ & 81 & 1.22 $\times$ $10^{-2}$ \\
6 & 6.37 $\times$ $10^{-7}$ & 243 & 1.55 $\times$ $10^{-4}$ \\
\bottomrule
\end{tabular*} 
\label{tab:nasbench-probs}

\end{minipage}

\end{table*} 

\section{Best practices checklist for NAS research} \label{app:checklist}
The area of NAS has seen problems with reproducibility, 
as well as fair empirical comparisons.
Following calls for fair and reproducible NAS research \cite{randomnas, nasbench}, a best practices checklist was recently created \cite{lindauer2019best}.
In order to promote fair and reproducible NAS research, we address all points on the checklist, and we encourage future papers to do the same.
Our code is available at \url{https://github.com/naszilla/naszilla}.

\begin{itemize}
    \item \emph{Code for the training pipeline used to evaluate the final architectures.}
    We used three of the most popular search spaces in NAS research, the NASBench-101 and NASBench-201 search spaces, and the DARTS search space.
    For NASBench-101 and 201, the accuracy of all architectures were precomputed.
    For the DARTS search space, we released our fork of the DARTS repo, which is forked from the DARTS repo designed specifically for reproducible experiments~\cite{randomnas}, making trivial changes to account for pytorch $1.2.0$.
    \item \emph{Code for the search space.}
    We used the popular and publicly avaliable NASBench and DARTS search spaces with no changes.
    \item \emph{Hyperparameters used for the final evaluation pipeline, as well as random seeds.}
    We left all hyperparameters unchanged.
    We trained the architectures found by BANANAS, ASHA, and DARTS five times each,
    using random seeds 0, 1, 2, 3, 4.

    \item \emph{For all NAS methods you compare, did you use exactly the same NAS benchmark, including the same dataset, search space, and code for training the architectures and hyperparameters for that code?}
    Yes, we did this by virtue of the NASBench-101 and 201 datasets.
    For the DARTS experiments, we used the reported architectures (found using the same search space and dataset as our method),
    and then we trained the final architectures using the same code, including hyperparameters.
    We compared different NAS methods using exactly the same NAS benchmark.
    \item \emph{Did you control for confounding factors?}
    Yes, we used the same setup for all of our NASBench-101 and 201 experiments. 
    For the DARTS search space, we compared our algorithm to two other algorithms using the same setup (pytorch version, CUDA version, etc).
    Across training over 5 seeds for each algorithm, we used different GPUs, 
    which we found to have no greater effect than using a different random seed.
    \item \emph{Did you run ablation studies?}
    Yes, in fact, ablation studies guided our entire decision process in constructing
    BANANAS. Section~\ref{sec:methodology} is devoted entirely to ablation studies.
    \item \emph{Did you use the same evaluation protocol for the methods being compared?}
    Yes, we used the same evaluation protocol for all methods and we tried multiple evaluation protocols.
    \item \emph{Did you compare performance over time?}
    Yes, all of our plots are performance over time.
    \item \emph{Did you compare to random search?}
    Yes.
    \item \emph{Did you perform multiple runs of your experiments and report seeds?}
    We ran 200 trials of our NASBench-101 and 201 experiments. Since we ran so many trials, we did not report random seeds.
    We ran four total trials of BANANAS on the DARTS search space. 
    Currently we do not have a fully deterministic version of BANANAS on the DARTS search space 
    (which would be harder to implement as the algorithm runs on 10 GPUs).
    However, the average final error across trials was within 0.1\%.
    
    \item \emph{Did you use tabular or surrogate benchmarks for in-depth evaluations} 
    Yes, we used NASBench-101 and 201.
    \item \emph{Did you report how you tuned hyperparameters, and what time and resources this required?}
    We performed light hyperparameter tuning at the start of this project, for
    the number of layers, layer size, learning rate, and number of epochs of the meta neural
    network. 
    We did not perform any hyperparameter tuning when we ran the algorithm on NASBench-201 for
    all three datasets, or the DARTS search space.
    This suggests that the current hyperparameters work well for most
    new search spaces.
    \item \emph{Did you report the time for the entire end-to-end NAS method?}
    We reported time for the entire end-to-end NAS method.
    \item \emph{Did you report all details of your experimental setup?}
    We reported all details of our experimental setup.

\end{itemize}

\end{document}